\documentclass[10pt,journal,compsoc]{IEEEtran}

\ifCLASSOPTIONcompsoc
 \usepackage[nocompress]{cite}
 \else
\usepackage{cite}
\fi

\usepackage{amsmath}

\usepackage{array}

\usepackage{mathtools}
\usepackage{bbold}
\usepackage{bm}
\usepackage{epigraph}
\usepackage{cuted} 
\usepackage{lipsum}

\usepackage[sort]{natbib}
\usepackage{multirow}
\setcitestyle{square}

\RequirePackage{epsfig}
\RequirePackage{amssymb}
\RequirePackage{natbib}
\RequirePackage{graphicx}

\usepackage{xurl}

\newcommand{\scalbijection}{h}
\newcommand{\actfunction}{h}

\newcommand{\obsvari}{Y}
\newcommand{\obsvar}{\mathbf{\obsvari}}
\newcommand{\obsvali}{y}
\newcommand{\obsval}{\mathbf{\obsvali}}
\newcommand{\obsset}{\mathcal{\obsvari}}
\newcommand{\obsdensity}{p_\obsvar}
\newcommand{\obsdata}{\mathcal{D}}
\newcommand{\obsnum}{M}

\newcommand{\couplingfunc}{\mathbf{\scalbijection}}
\newcommand{\couplingfunci}{\scalbijection}

\newcommand{\flowvali}{x}
\newcommand{\flowval}{\mathbf{\flowvali}}
\newcommand{\flowdim}{D}
\newcommand{\flownum}{N}

\newcommand{\flowfunc}{\mathbf{g}}
\newcommand{\flowinv}{\mathbf{f}}

\newcommand{\jac}[1]{\textrm{D}#1}
\newcommand{\flowfuncJ}{{\jac{\flowfunc}}}
\newcommand{\flowinvJ}{{\jac{\flowinv}}}
\newcommand{\flowparams}{\theta}

\newcommand{\basevari}{Z}
\newcommand{\basevar}{\mathbf{\basevari}}
\newcommand{\basevali}{z}
\newcommand{\baseval}{\mathbf{\basevali}}
\newcommand{\baseset}{\mathcal{\basevari}}
\newcommand{\basedensity}{p_\basevar}
\newcommand{\baseparams}{\phi}

\newcommand{\R}{\mathbb{R}}
\newcommand{\C}{\mathbb{C}}
\newcommand{\E}{\mathbb{E}}
\newcommand{\bigO}{\mathcal{O}}

\newcommand{\diag}{\textrm{diag}}
\newcommand{\tr}{\textrm{Tr}}

\newcommand{\KL}{D_{KL} }
\newcommand{\placeholder}{\boldsymbol{\cdot} }
\newcommand{\NN}{\text{NN} }
\newcommand{\eg}{\emph{e.g.}}
\newcommand{\ie}{\emph{i.e.}}

\newcommand{\BlackBox}{\rule{1.5ex}{1.5ex}}  
\newenvironment{proof}{\par\noindent{\bf Proof\ }}{\hfill\BlackBox\\[2mm]}
 
\newtheorem{theorem}{Theorem}
 
\newtheorem{proposition}[theorem]{Proposition} 
\newtheorem{remark}[theorem]{Remark}

\newtheorem{definition}[theorem]{Definition}

\begin{document}

\title{Normalizing Flows: An Introduction and Review of Current Methods}
%
%
%
%

\author{Ivan~Kobyzev,~\IEEEmembership{}
        Simon~J.D.~Prince,~\IEEEmembership{}
        and~Marcus~A.~Brubaker.~\IEEEmembership{Member, IEEE}
\IEEEcompsocitemizethanks{\IEEEcompsocthanksitem Borealis AI, Canada.\protect\\
 ivan.kobyzev@borealisai.com\protect\\
simon.prince@borealisai.com\protect\\
mab@eecs.yorku.ca}
}

%
%

\markboth{IEEE Transactions on Pattern Analysis and Machine Intelligence}%
{Kobyzev \MakeLowercase{\textit{et al.}}: Normalizing Flows: An Introduction and Review of Current Methods}

%



\IEEEtitleabstractindextext{%
\begin{abstract}
Normalizing Flows are generative models which produce tractable distributions where both sampling and density evaluation can be efficient and exact. The goal of this survey article is to give a coherent and comprehensive review of the literature around the construction and use of Normalizing Flows for distribution learning. We aim to provide context and explanation of the models, review current state-of-the-art literature, and identify open questions and promising future directions.
\end{abstract}

}

\maketitle


\IEEEdisplaynontitleabstractindextext

%
\IEEEpeerreviewmaketitle


%
%
%
%

\IEEEraisesectionheading{\section{Introduction}\label{sec:introduction}}
\IEEEPARstart{A}{major}  goal of statistics and machine learning has been to model a probability distribution given samples drawn from that distribution.  This is an example of unsupervised learning and is sometimes called generative modelling.  Its importance derives from the relative abundance of unlabelled data compared to labelled data.  Applications include density estimation, outlier detection,  prior construction, and dataset summarization.

 Many methods for generative modeling have been proposed.  Direct analytic approaches approximate observed data with a fixed family of distributions. Variational approaches and expectation maximization introduce latent variables to explain the observed data. They provide additional flexibility but can increase the complexity of learning and inference.  Graphical models \citep{graphModels} explicitly model the conditional dependence between random variables.  Recently, generative neural approaches have been proposed including generative adversarial networks (GANs) \citep{Goodfellow2014GenerativeAN} and variational auto-encoders (VAEs) \citep{Kingma2014AutoEncodingVB}.

GANs and VAEs have demonstrated impressive performance results on challenging tasks such as learning distributions of natural images. However, several issues limit their application in practice. Neither allows for exact evaluation of the probability density of new points. Furthermore, training can be challenging due to a variety of phenomena including mode collapse, posterior collapse, vanishing gradients and training instability \citep{Bowman2015GeneratingSF, Salimans2016ImprovedTF}. 
 
Normalizing Flows (NF) are a family of generative models with tractable distributions where both sampling and density evaluation can be efficient and exact.  Applications include image generation \citep{Kingma2018, ho2019flow}, noise modelling \citep{Abdelhamed2019}, video generation \citep{videoflow}, audio generation \citep{flowwavenet, waveglow, esling2019}, graph generation \citep{Madhawa2019}, reinforcement learning \citep{Ward2019RL, Touati2019, Mazoure2019}, computer graphics \citep{Muller2018NeuralSampling}, and physics \citep{Noe2019, Koller2019Eq, KazeWong2020, Wirnsberger2020, Kanwar2020}. 
 
There are several survey papers for VAEs \citep{Kingma2019IntroVAE} and GANs \citep{Creswell2018GenerativeAN, Wang2017GenerativeAN}.  This article aims to provide a comprehensive review of the literature around Normalizing Flows for distribution learning. Our goals are to
  1) provide context and explanation to enable a reader to become familiar with the basics,
  2) review the current literature, and 3) identify open questions and promising future directions.
  Since this article was first made public, an excellent complementary treatment has been provided by \cite{FlowSurvey2019}.  Their article is more tutorial in nature and provides many details concerning implementation,
  whereas our treatment is more formal and focuses mainly on the families of flow models.

In Section \ref{background}, we introduce Normalizing Flows and describe how they are trained. In Section \ref{Methods} we review constructions for Normalizing Flows. In Section \ref{performance} we describe datasets for testing Normalizing Flows and discuss the performance of different approaches. Finally, in Section \ref{discussion} we discuss open problems and possible research directions.

\section{Background}\label{background}
Normalizing Flows were popularised by \citet{Rezende2015} in the context of variational inference and by \citet{Dinh2015} for density estimation.  However, the framework was previously defined in \citet{Tabak2010} and \citet{Tabak2013}, and explored for clustering and classification \citep{Agnelli2010}, and density estimation \citep{Rippel2013, Laurence2014}.

A Normalizing Flow is a transformation of a simple probability distribution (\eg, a standard normal) into a more complex distribution by a sequence of invertible and differentiable mappings. The density of a sample can be evaluated by transforming it back to the original simple distribution and then computing the product of i) the density of the inverse-transformed sample under this distribution and ii) the associated change in volume induced by the sequence of inverse transformations.   The change in volume is the product of the absolute values of the determinants of the Jacobians for each transformation, as required by the change of variables formula.  

The result of this approach is a mechanism to construct new families of distributions by choosing an initial density and then chaining together some number of parameterized, invertible and differentiable transformations. The new density can be sampled from (by sampling from the initial density and applying the transformations) and the density at a sample (\ie, the likelihood) can be computed as above.

\subsection{Basics}

 Let $\basevar \in \R^\flowdim$ be a random variable with a known and tractable probability density function $\basedensity : \R^\flowdim \rightarrow \R$. Let $\flowfunc$ be an invertible function and $\obsvar = \flowfunc(\basevar)$. Then using the change of variables formula, one 
 can compute the probability density function of the random variable $\obsvar$:
 \begin{align}
\label{maxloglike}
 \obsdensity(\obsval) &
 =  \basedensity(\flowinv(\obsval)) \left| \det \flowinvJ(\obsval) \right|\nonumber \\
& =  \basedensity(\flowinv(\obsval)) \left| \det \flowfuncJ(\flowinv(\obsval)) \right|^{-1} ,
\end{align}
 where $\flowinv$ is the inverse of $\flowfunc$,  $\flowinvJ(\obsval) = \frac{\partial \flowinv} {\partial \obsval}$ is the Jacobian of $\flowinv$ and $\flowfuncJ(\baseval) = \frac{\partial \flowfunc} {\partial \baseval}$ is the Jacobian of $\flowfunc$.
 This new density function $\obsdensity(\obsval)$ is called a \textit{pushforward} of the density $\basedensity$ by the function $\flowfunc$ and denoted by $\flowfunc_* \basedensity $ (Figure \ref{fig:change_of_var}).

 \begin{figure}[t!]
\centering
\includegraphics[scale=0.34]{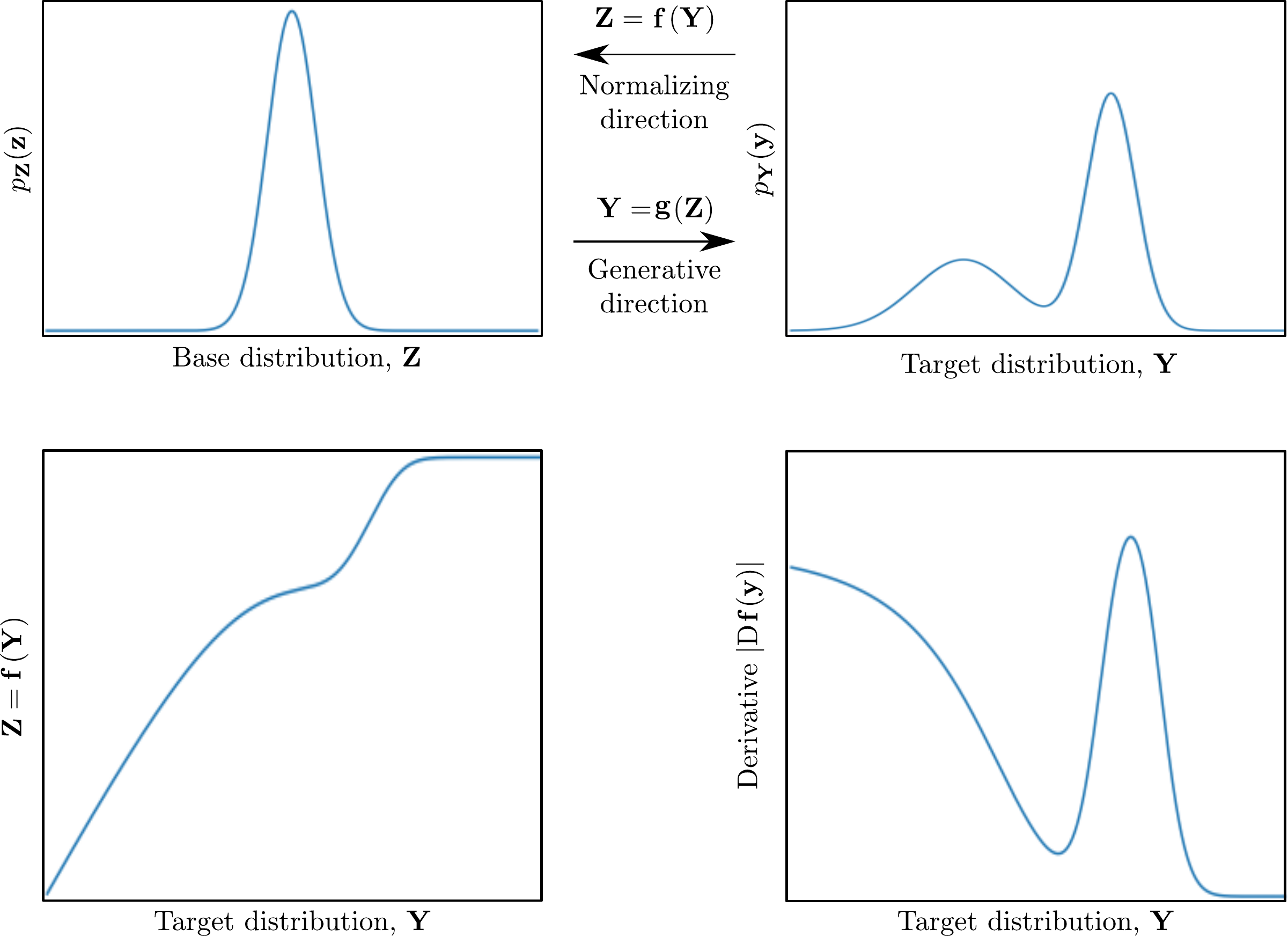}
\caption{Change of variables (Equation \eqref{maxloglike}). Top-left: the density of the source $\basedensity$. Top-right: the density function of the target  distribution $\obsdensity(\obsval)$. There exists a bijective function $\flowfunc$, such that $\obsdensity = \flowfunc_* \basedensity $, with inverse $\flowinv$.
Bottom-left: the inverse function $\flowinv$. Bottom-right:  the absolute Jacobian (derivative) of $\flowinv$.}
\label{fig:change_of_var}
\end{figure}

In the context of generative models, 
the above function $\flowfunc$ (a generator) ``pushes forward" the base density $\basedensity$ (sometimes referred to as the ``noise'')  to a more complex density. This movement from base density to final complicated density is the \textit{generative direction}. Note that to generate a data point $\obsval$, one can sample $\baseval$ from the base distribution, and then apply the generator: $\obsval = \flowfunc(\baseval) $.

The inverse function $\flowinv$ moves (or ``flows'') in the opposite, \textit{normalizing direction}: from a complicated and irregular data distribution towards the simpler, more regular or ``normal'' form, of the base measure $\basedensity$.
This view is what gives rise to the name ``normalizing flows'' as $\flowinv$
is ``normalizing'' the data distribution.
This term is doubly accurate if the base measure $\basedensity$ is chosen as a Normal distribution as it often is in practice.

Intuitively, if the transformation $\flowfunc$ can be arbitrarily complex, one can generate any distribution $\obsdensity$ from any base distribution $\basedensity$ under reasonable assumptions on the two distributions.
This has been proven formally \citep{transportation, triangular2005, triangular2008}. See Section \ref{s:universal}.

Constructing arbitrarily complicated non-linear invertible functions (bijections) can be difficult.
By the term \textit{Normalizing Flows} people mean bijections which are convenient to compute, invert, and calculate the determinant of their Jacobian.
One approach to this is to note that the composition of invertible functions is itself invertible and the determinant of its Jacobian has a specific form.
In particular, let $\flowfunc_1, \dots, \flowfunc_\flownum$ be a set of $\flownum$ bijective functions and define $\flowfunc = \flowfunc_\flownum \circ \flowfunc_{\flownum-1} \circ \dots \circ \flowfunc_1$  to be the composition of the functions.
Then it can be shown that $\flowfunc$ is also bijective, with inverse:
\begin{equation}
\flowinv = \flowinv_1 \circ \dots \circ \flowinv_{\flownum-1} \circ \flowinv_\flownum,
\end{equation}
and the determinant of the Jacobian is 
\begin{equation}
\det \flowinvJ(\obsval) = \prod_{i=1}^\flownum \det \flowinvJ_i(\flowval_{i}),
\end{equation}
where
$\flowinvJ_i(\obsval) = \frac{\partial \flowinv_i} {\partial \flowval}$ is the Jacobian of $\flowinv_i$.  We denote the value of the $i$-th intermediate flow as
$\flowval_i = \flowfunc_i \circ \dots \circ \flowfunc_1(\baseval) = \flowinv_{i+1} \circ \dots \circ \flowinv_{\flownum}(\obsval)$ and so $\flowval_{\flownum} = \obsval$. Thus, a set of nonlinear bijective functions can be composed to construct successively more complicated functions.

\subsubsection{More formal construction}
\label{formal_def}
In this section we explain normalizing flows from more formal perspective. Readers unfamiliar with measure theory can safely skip to Section \ref{s:applications}. First, let us recall the general definition of a pushforward.

\begin{definition}
\label{d:pushforward}
If $(\baseset,\Sigma_\baseset)$, $(\obsset,\Sigma_\obsset) $ are measurable spaces,   $\flowfunc$ is a measurable mapping between them, and $\mu$ is a measure on $\baseset$,  then one can define a measure on $\obsset$ (called the pushforward measure and denoted by $\flowfunc_* \mu$) by the formula
\begin{equation}
\flowfunc_* \mu (U) = \mu (\flowfunc^{-1}(U) ), \quad \text{for all} \ U \in \Sigma_\obsset.
\end{equation}
\end{definition}
This notion gives a general formulation of a generative model.  Data can be understood as a  sample from  a measured ``data'' space $(\obsset,\Sigma_\obsset, \nu) $, which we want to learn. To do that  one can introduce a simpler measured space $(\baseset,\Sigma_\baseset, \mu)$ and find a function $\flowfunc: \baseset \to \obsset$, such that $\nu = \flowfunc_*\mu $. This function $\flowfunc$ can be interpreted as a ``generator", and $\baseset $ as a latent space. This view puts generative models in the context of transportation theory \citep{transportation}.

In this survey we will assume that $\baseset = \R^\flowdim$,  all sigma-algebras are Borel, and all measures are absolutely continuous  with respect to Lebesgue measure (\ie, $\mu = \basedensity \mathbf{d z}  $).

\begin{definition}
\label{d:diffeo}
A function $\flowfunc : \R^\flowdim \rightarrow \R^\flowdim$ is called a diffeomorphism, if it is bijective, differentiable, and its inverse is differentiable as well. 
\end{definition}

The pushforward of an absolutely continuous measure  $\basedensity \mathbf{d z}$ by a diffeomorphism  $\flowfunc$ is also absolutely continuous with a density function given by Equation \eqref{maxloglike}. 
Note that this more general approach is important for studying generative models on non-Euclidean spaces (see Section \ref{s:nonEuclid}).

\begin{remark}
It is common in the normalizing flows literature to simply refer to diffeomorphisms as ``bijections'' even though this is formally incorrect.
In general, it is not necessary that $\flowfunc$ is everywhere differentiable;  rather it is sufficient that it is differentiable only almost everywhere with respect to the Lebesgue measure on $\R^\flowdim $.
This allows, for instance, piecewise differentiable functions to be used in the construction of $\flowfunc$.
\end{remark}

\subsection{Applications}\label{s:applications}
\subsubsection{Density estimation and sampling}
\label{s:Density_Estimation}
The natural and most obvious use of normalizing flows is to perform density estimation.
For simplicity assume that only a single flow, $\flowfunc$, is used and it is parameterized by the vector $\flowparams$.
Further, assume that the base measure, $\basedensity$ is given and is parameterized by the vector $\baseparams$.
Given a set of data observed from some complicated distribution, $\obsdata = \{ \obsval^{(i)} \}_{i=1}^{\obsnum}$, we can then perform likelihood-based estimation of the parameters $ \Theta=(\flowparams,\baseparams)$.
The data likelihood in this case simply becomes
\begin{eqnarray}
\label{maxloglike2}
\log p(\obsdata | \Theta)
& = & \sum_{i=1}^{\obsnum} \log \obsdensity(\obsval^{(i)} | \Theta) \\
& = & \sum_{i=1}^{\obsnum} \log \basedensity(\flowinv(\obsval^{(i)} | \flowparams) | \baseparams) + \log \left| \det \flowinvJ(\obsval^{(i)} | \flowparams) \right| \nonumber 
\end{eqnarray}
where the first term is the log likelihood of the sample under the base measure and the second term, sometimes called the log-determinant or volume correction, accounts for the change of volume induced by the transformation of the normalizing flows (see Equation \eqref{maxloglike}). During training, the parameters of the flow ($\flowparams$) and of the base distribution ($\baseparams$) are adjusted to maximize the log-likelihood. 

Note that evaluating the likelihood of a distribution modelled by a normalizing flow requires computing $\flowinv$ (\ie, the normalizing direction), as well as its log determinant.  The efficiency of these operations is particularly important during training where the likelihood is repeatedly computed.
However, sampling from the distribution defined by the normalizing flow requires evaluating the inverse $\flowfunc$ (\ie, the generative direction).
Thus sampling performance is determined by the cost of the generative direction.
Even though a flow must be theoretically invertible, computation of the inverse may be difficult in practice; hence, for density estimation it is common to model a flow in the normalizing direction (\ie, $\flowinv$). \footnote{To ensure both efficient density estimation and sampling, \cite{Oord2017ParallelWF} proposed an approach called Probability Density Distillation which trains the flow $\flowinv$ as normal and then uses this as a teacher network to train a tractable student network  $\flowfunc$. }

Finally, while maximum likelihood estimation is often effective (and statistically efficient under certain conditions) other forms of estimation can and have been used with normalizing flows.
In particular, adversarial losses can be used with normalizing flow models (\eg, in Flow-GAN \citep{Grover2018}).

\subsubsection{Variational Inference}

Consider a latent variable model $p(\mathbf{x})=\int p(\mathbf{x},\obsval)d\obsval$ where $\mathbf{x}$ is an observed variable and $\obsval $ the latent variable.
The posterior distribution $p( \obsval | \mathbf{x})$ is used when estimating the parameters of the model, but its computation is usually intractable in practice.  One approach is to use variational inference and introduce the approximate posterior $q (\obsval | \mathbf{x}, \flowparams) $ where $\flowparams$ are parameters of the variational distribution.
Ideally this distribution should be as close to the real posterior as possible.
This is done by minimizing the KL divergence $\KL(q (\obsval | \mathbf{x}, \flowparams) ||p(\obsval | \mathbf{x}))$, which is equivalent to maximizing the evidence lower bound $\mathcal{L}(\flowparams) = \E_{q (\obsval | \mathbf{x}, \flowparams)}[\log(p(\obsval , \mathbf{x})) - \log(q (\obsval | \mathbf{x}, \flowparams))]$. The latter optimization can be done with gradient descent; however for that one needs to compute gradients of the form $\nabla_\flowparams \E_{q (\obsval | \mathbf{x}, \flowparams)}[h(\obsval)]$, which is not straightforward.

As was observed by \citet{Rezende2015}, one can reparametrize $q (\obsval| \mathbf{x}, \flowparams) = \obsdensity(\obsval|\flowparams)$ with normalizing flows.
Assume for simplicity, that only a single flow $\flowfunc$ with parameters $\flowparams$ is used, $\obsval = \flowfunc(\baseval|\flowparams)$ and the base distribution $\basedensity(\baseval)$ does not depend on $\flowparams$. 
Then
\begin{eqnarray}
\E_{\obsdensity(\obsval|\flowparams)}[h(\obsval)] = \E_{\basedensity(\baseval)}[h(\flowfunc(\baseval|\flowparams))],
\end{eqnarray}
and  the gradient of the right hand side with respect to $\flowparams$ can be computed.  
This approach generally to computing gradients of an expectation is often called the ``reparameterization trick''.

In this scenario evaluating the likelihood is only required at points which have been sampled. Here the sampling performance and evaluation of the log determinant are the only relevant metrics and computing the inverse of the mapping may not be necessary.
Indeed, the planar and radial flows introduced in \citet{Rezende2015} are not easily invertible (see Section \ref{s:planar_radial}).

\section{Methods}
\label{Methods}
Normalizing Flows should satisfy several conditions in order to be practical.  They should:
\begin{itemize}
    \item be invertible; for sampling we need $\flowfunc$ while for computing likelihood we need $\flowinv$,
    \item be sufficiently expressive to model the distribution of interest,
    \item be computationally efficient, both in terms of computing $\flowinv$ and $\flowfunc$ (depending on the application) but also in terms of the calculation of the determinant of the Jacobian.
\end{itemize}
In the following section, we describe different types of flows and comment on the above properties.
An overview of the methods discussed can be seen in figure \ref{fig:flowverview}.

 \begin{figure}[t!]
\centering
\includegraphics[scale=0.45]{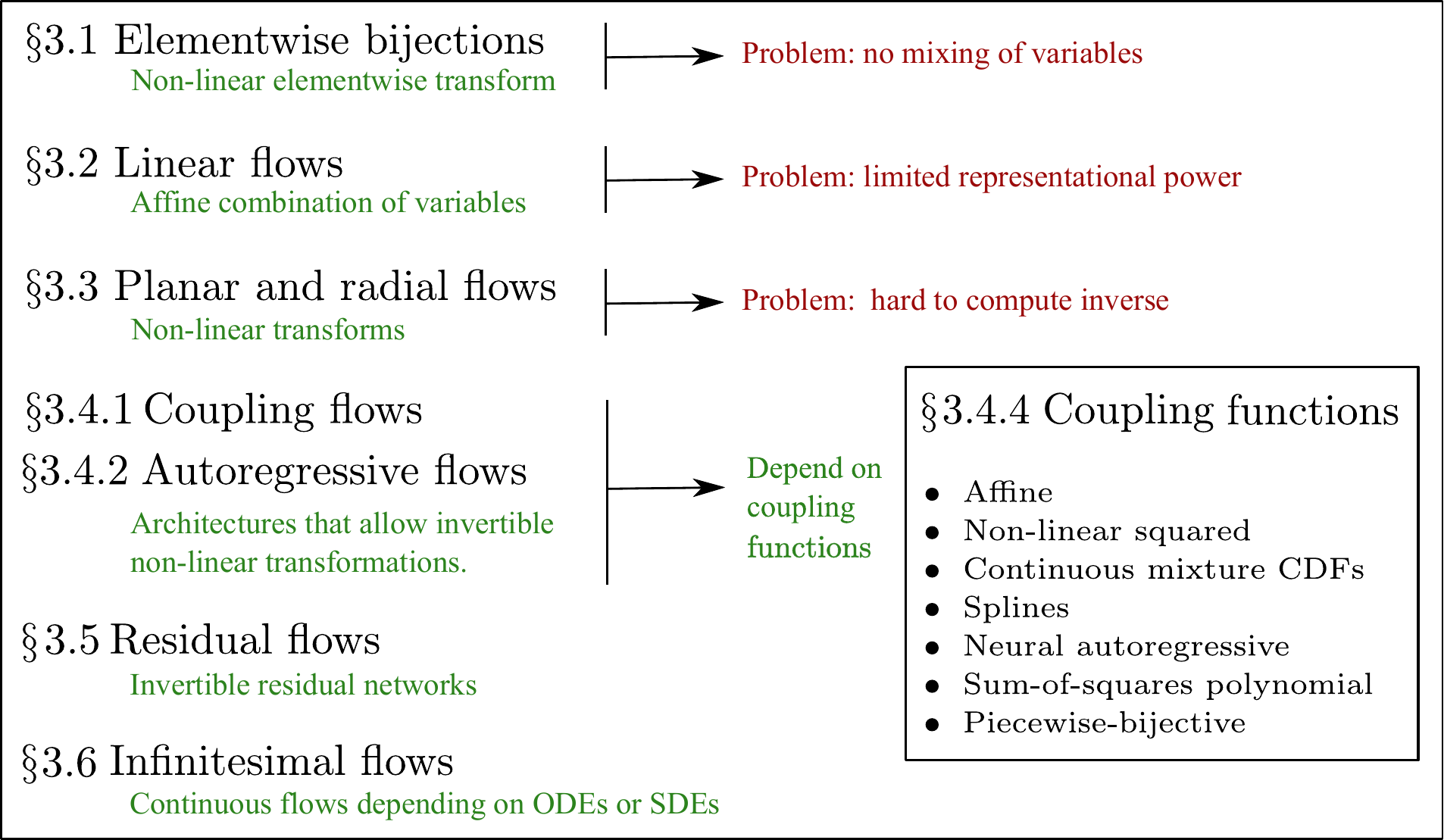}
\caption{Overview of flows discussed in this review.  We start with elementwise bijections, linear flows, and planar and radial flows. All of these have drawbacks and are limited in utility. We then discuss two architectures (coupling flows and autoregressive flows) which support invertible non-linear transformations.
These both use a coupling function, and we summarize the different coupling functions available.
Finally, we discuss residual flows and their continuous extension infinitesimal flows.  }
\label{fig:flowverview}
\end{figure}

\subsection{Elementwise Flows}
A basic form of bijective non-linearity can be constructed given any bijective scalar function.
That is, let $\scalbijection : \R \rightarrow \R$ be a scalar valued bijection.
Then, if $\flowval = (\flowvali_1,\flowvali_2,\dots,\flowvali_\flowdim)^T$,
\begin{equation}
\flowfunc(\flowval) = (\scalbijection(\flowvali_1),\scalbijection(\flowvali_2),\dots,\scalbijection(\flowvali_\flowdim))^T
\end{equation}
is also a bijection whose inverse simply requires computing $\scalbijection^{-1}$ and whose Jacobian is the product of the absolute values of the derivatives of $\scalbijection$.
This can be generalized by allowing each element to have its own distinct bijective function which might be useful if we wish to only modify portions of our parameter vector.
In deep learning terminology, $\scalbijection$, could be viewed as an ``activation function''.
Note that the most commonly used activation function ReLU is not bijective and can not be directly applicable,
however, the (Parametric) Leaky ReLU \citep{Maas2013,He2015} can be used instead among others.
Note that recently spline-based activation functions have also been considered \citep{Durkan2019Cubic,Durkan2019Neural} and will be discussed in Section \ref{s:spline}.

\subsection{Linear Flows}

Elementwise operations alone are insufficient as they cannot express any form of correlation between dimensions. Linear mappings can express correlation between dimensions:
\begin{equation}
\flowfunc(\flowval) = \mathbf{A} \flowval + \mathbf{b}
\end{equation}

\noindent where $\mathbf{A}\in \R^{\flowdim \times \flowdim}$ and $\mathbf{b} \in \R^\flowdim$ are parameters. If $\mathbf{A}$ is an invertible matrix, the function is invertible. 

Linear flows are limited in their expressiveness. Consider a Gaussian base distribution: $\basedensity(\baseval) = \mathcal{N}(\baseval, \mathbf{\mu}, \Sigma)$. After transformation by a linear flow, the distribution remains Gaussian with distribution $\obsdensity = \mathcal{N}(\obsval, \mathbf{A}\mu + \mathbf{b}, \mathbf{A}^T\Sigma \mathbf{A} )$. More generally, a linear flow of a distribution from the exponential family remains in the exponential family.
However, linear flows are an important building block as they form the basis of affine coupling flows (Section \ref{s:affine}).

 Note that the determinant of the Jacobian is simply $\det(\mathbf{A})$, which can be computed in $\bigO (\flowdim^3)$, as can the inverse.  Hence, using linear flows can become expensive for large $\flowdim$.  
By restricting the form of $\mathbf{A}$ we can avoid these practical problems at the expense of expressive power.  In the following sections we discuss different ways of limiting the form of linear transforms to make them more practical.

\subsubsection{Diagonal}

If $\mathbf{A}$ is diagonal with nonzero diagonal entries, then its inverse can be computed in linear time and its determinant is the product of the diagonal entries.  However, the result is an elementwise transformation and hence cannot express correlation between dimensions.  Nonetheless, a diagonal linear flow can still be useful for representing normalization transformations \citep{Dinh2017} which have become a ubiquitous part of modern neural networks \citep{Ioffe2015}.

\subsubsection{Triangular}
\label{s:lintriang}
The triangular matrix is a more expressive form of linear transformation whose  determinant is the product of its diagonal. It is non-singular so long as its diagonal entries are non-zero. Inversion is relatively inexpensive requiring a single pass of back-substitution costing $\bigO(\flowdim^2)$ operations.

\cite{Tomczakconvex} combined $K$ triangular matrices $\mathbf{T}_i$, each with ones on the diagonal, and a $K$-dimensional probability vector $\omega$ to define a more general linear flow  $\obsval = (\sum_{i=1}^K \omega_i \mathbf{T}_i ) \baseval$. The determinant of this bijection is one. However finding the inverse has $\bigO(\flowdim^3)$ complexity, if some of the matrices are upper- and some are lower-triangular.

\subsubsection{Permutation and Orthogonal}
\label{s:orthogonal}

The expressiveness of triangular transformations is sensitive to the ordering of dimensions.  Reordering the dimensions can be done easily using a permutation matrix which has an absolute determinant of $1$.
Different strategies have been tried, including  reversing and a fixed random permutation \citep{Dinh2017,Kingma2018}.
However, the permutations cannot be directly optimized and so remain fixed after initialization which may not be optimal.

A more general alternative is the use of orthogonal transformations.
The inverse and absolute determinant of an orthogonal matrix are both trivial to compute which make them efficient.
\citet{Tomczak2016} used orthogonal matrices parameterized by the \textit{Householder transform}. The idea
is based on the fact from linear algebra  that any orthogonal matrix can be written as a product of reflections.  
To parameterize a reflection matrix $H$ in $\R^\flowdim$ one fixes a nonzero vector $\mathbf{v} \in \R^\flowdim$, and then defines $H = \mathbb{1} - \frac{2}{||\mathbf{v}||^2}\mathbf{v}\mathbf{v}^T$.

\subsubsection{Factorizations}
Instead of limiting the form of $\mathbf{A}$, \cite{Kingma2018} proposed using the $LU$ factorization:
\begin{equation}
\flowfunc(\flowval) = \mathbf{P}\mathbf{L}\mathbf{U} \flowval + \mathbf{b}
\end{equation}
where $\mathbf{L}$ is lower triangular with ones on the diagonal, $\mathbf{U}$ is upper triangular with non-zero diagonal entries, and $\mathbf{P}$ is a permutation matrix.
The determinant is the product of the diagonal entries of $\mathbf{U}$ which can be computed in $\bigO (\flowdim)$.
The inverse of the function $\flowfunc$ can be computed using two passes of backward substitution in $\bigO (\flowdim^2)$.
However, the discrete permutation $\mathbf{P}$ cannot be easily optimized.
To avoid this, $\mathbf{P}$ is randomly generated initially and then fixed. 
\cite{Hoogeboom2019} noted that fixing the permutation matrix limits the flexibility of the transformation, and proposed using the $QR$ decomposition instead where the orthogonal matrix $Q$ is described with Householder transforms.

\subsubsection{Convolution}

Another form of linear transformation is a convolution which has been a core component of modern deep learning architectures.
While convolutions are easy to compute their inverse and determinant are non-obvious.
Several approaches have been considered.
\cite{Kingma2018} restricted themselves to ``$1\times 1$'' convolutions for flows which are simply a full linear transformation but applied only across channels.
\cite{Zheng2018} used 1D convolutions (\textbf{ConvFlow}) and exploited the triangular structure of the resulting transform to efficiently compute the determinant.
However \cite{Hoogeboom2019} 
have provided a more general solution for modelling $d\times d$ convolutions, either by stacking together masked autoregressive convolutions (referred to as Emerging Convolutions) or by exploiting the Fourier domain representation of convolution to efficiently compute inverses and determinants (referred to as Periodic Convolutions).

\subsection{Planar and Radial Flows}\label{s:planar_radial}

\cite{Rezende2015} introduced planar and radial flows.  They are relatively simple, but their inverses aren't easily computed. These flows are not widely used in practice, yet  they are reviewed here for completeness.

\subsubsection{Planar Flows}
Planar flows expand and contract the distribution along certain specific directions and take the form
\begin{equation}
\flowfunc(\flowval) = \flowval + \mathbf{u} \actfunction(\mathbf{w}^T \flowval + b),
\end{equation}
where $\mathbf{u},\mathbf{w} \in \R^\flowdim$ and $b \in \R$ are parameters and $\actfunction: \R \to \R$ is a smooth non-linearity.
The Jacobian determinant for this transformation is 
\begin{eqnarray}
\det \left(\frac{\partial \flowfunc}{\partial \flowval}\right) & = & \det (\mathbb{1}_\flowdim  + \mathbf{u} \actfunction'(\mathbf{w}^T \flowval + b)\mathbf{w}^T) \nonumber \\ & = &1 + \actfunction'(\mathbf{w}^T \flowval + b)\mathbf{u}^T\mathbf{w},
\end{eqnarray}
where the last equality comes from the application of the matrix determinant lemma. This can be computed in $\bigO (\flowdim)$ time. The inversion of this flow isn't possible in closed form and may not exist for certain choices of $\actfunction(\cdot)$ and certain parameter settings \citep{Rezende2015}.

The term $\mathbf{u} \actfunction(\mathbf{w}^T \flowval + b) $ can be interpreted as a multilayer perceptron with a bottleneck hidden layer with a single unit \citep{Kingma2016}.
This bottleneck means that one needs to stack many planar flows to get high expressivity. 
\cite{Hasenclever2017} and \cite{sylvester} introduced \textbf{Sylvester flows} to resolve this problem:
\begin{equation}
    \flowfunc(\flowval) = \flowval + \mathbf{U} \couplingfunc(\mathbf{W}^T \flowval + \mathbf{b}),
\end{equation}
where $\mathbf{U}$ and $\mathbf{W}$ are $\flowdim \times M $ matrices, $\mathbf{b}\in \R^M$ and $\couplingfunc: \R^M \to \R^M$ is an elementwise smooth nonlinearity, where $M \leq \flowdim$ is a hyperparameter to choose and which can be interpreted as the  dimension of a hidden layer.
In this case the Jacobian determinant is: 
\begin{eqnarray}
\hspace{-0.5cm}\det \left(\frac{\partial \flowfunc}{\partial \flowval}\right) \!\!&\!\! =\!\! &\!\! \det (\mathbb{1}_\flowdim  + \mathbf{U} \, \diag(\couplingfunc'(\mathbf{W}^T \flowval + b))\mathbf{W}^T) \nonumber \\\!\!&\!\! =\!\! &\!\! \det (\mathbb{1}_M  + \diag(\couplingfunc'(\mathbf{W}^T \flowval + b))\mathbf{W}\mathbf{U}^T  ),
\end{eqnarray}
where the last equality is Sylvester's determinant identity (which gives these flows their name). This can be computationally efficient if $M$ is small. Some sufficient conditions for the invertibility of Sylvester transformations are discussed in \citet{Hasenclever2017} and \citet{sylvester}.

\subsubsection{Radial Flows}
Radial flows instead modify the distribution around a specific point so that
\begin{equation}
\flowfunc(\flowval) = \flowval + \frac{\beta}{\alpha + \Vert \flowval - \flowval_0 \Vert} (\flowval - \flowval_0)
\end{equation}
where $\flowval_0 \in \R^\flowdim$ is the point around which the distribution is distorted, and $\alpha,\beta \in \R$ are parameters, $\alpha >0$.
As for planar flows, the Jacobian determinant can be computed relatively efficiently. 
The inverse of radial flows cannot be given in closed form but does exist under suitable constraints on the parameters \citep{Rezende2015}.

\subsection{Coupling and Autoregressive Flows}

In this section we describe coupling and auto-regressive flows which are the two most widely used flow architectures. We first present them in the general form, and then in Section \ref{s:coupling_functions} we give specific examples.

\subsubsection{Coupling Flows}
\label{coupling}
\cite{Dinh2015} introduced a coupling method to  enable highly expressive transformations for flows (Figure~\ref{fig:coupling}a). 
Consider a disjoint partition of the input $\flowval \in \R^\flowdim$ into two subspaces: 
$(\flowval^A,\flowval^B) \in \R^d \times \R^{\flowdim-d} $ and a bijection $\couplingfunc(\placeholder \, ; \theta): \R^d \to \R^d$, parameterized by $\theta$.
Then one can define a function $\flowfunc: \R^\flowdim \to \R^\flowdim$ by the formula:
\begin{align}
\label{eq:coupling}
 \obsval^A &= \couplingfunc(\flowval^A; \Theta(\flowval^B)) \nonumber\\
    \obsval^B &= \flowval^B ,
\end{align}
where the parameters $\theta$ are defined by {\em any} arbitrary function $\Theta(\flowval^B)$  which only uses $\flowval^B$ as input.
This function is called a \textit{conditioner}. 
The bijection $\couplingfunc$ is called a \textit{coupling function}, and the resulting function $\flowfunc$ is called a \textit{coupling flow}. A coupling flow is invertible if and only if $\couplingfunc$ is invertible and has inverse:
\begin{align}
 \flowval^A  &= \couplingfunc^{-1}(\obsval^A; \Theta(\flowval^B)) \nonumber\\
    \flowval^B &=  \obsval^B .
\end{align}
The Jacobian of $\flowfunc$ is a block triangular matrix where the diagonal  blocks are $\jac{\couplingfunc}$ and the identity matrix respectively.
Hence the determinant of the Jacobian of the coupling flow is simply the determinant of $\jac{\couplingfunc}$.

Most coupling functions are applied to  $\flowval^A$ element-wise:
\begin{equation}
\label{coupling_func}
\couplingfunc(\flowval^A ; \theta) = (\couplingfunci_1(x^A_1 ; \theta_1) , \dots, \couplingfunci_d(x^A_d ; \theta_d) ),
\end{equation}
where each $\couplingfunci_i(\placeholder ; \theta_i): \R \to \R $ is a scalar bijection.
In this case a coupling flow is a triangular transformation (\ie, has a triangular Jacobian matrix).
See Section \ref{s:coupling_functions} for examples.

The power of a coupling flow resides in the ability of a conditioner $\Theta(\flowval^B)$ to be arbitrarily complex. In practice it is usually modelled as a neural network. For example, \citet{Kingma2018} used a shallow ResNet architecture.

Sometimes, however, the conditioner can be constant (\ie, not depend on $\flowval^B$ at all).
This allows for the construction of a ``\textit{multi-scale flow}'' \citet{Dinh2017} which gradually introduces dimensions to the distribution in the generative direction (Figure \ref{fig:coupling}b).
In the normalizing direction, the dimension reduces by half after each iteration step, such that most of semantic information is retained.
This reduces the computational costs of transforming high dimensional distributions and can capture the multi-scale structure inherent in certain kinds of data like natural images.

\begin{figure}[t!]
\centering
\includegraphics[scale=0.44]{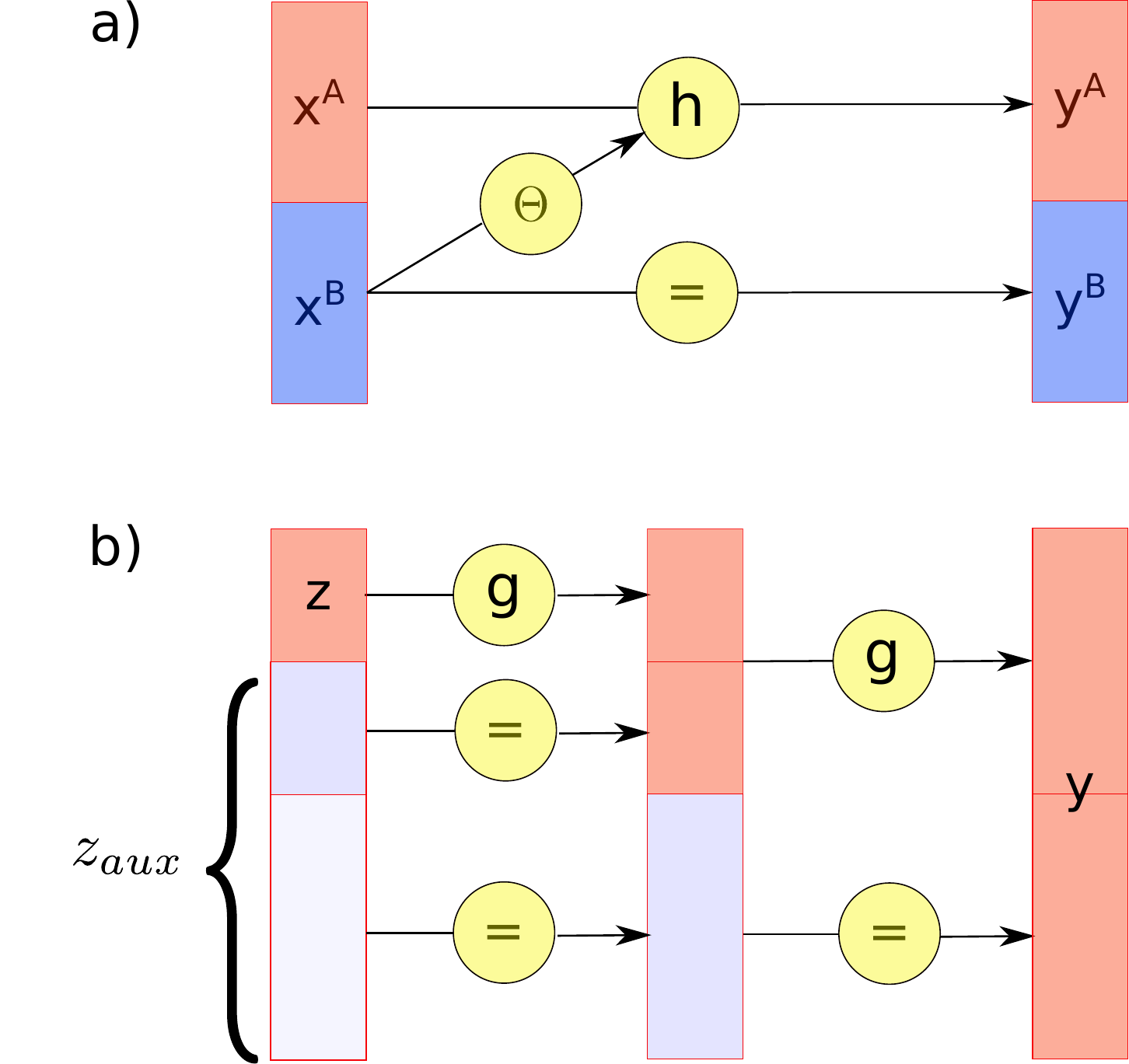}
\caption{Coupling architecture. a) A single coupling flow described in Equation \eqref{eq:coupling}.
A coupling function $\couplingfunc$ is applied to one part of the space, while its parameters depend on the other part.
b) Two subsequent multi-scale flows in the generative direction. A flow is applied to a relatively low dimensional vector $\baseval$; its parameters no longer depend on the rest part $\baseval_{aux}$. Then new dimensions are gradually introduced to the distribution.  }
\label{fig:coupling}
\end{figure}

The question remains of how to partition $\flowval$.
This is often done by splitting the dimensions in half \citep{Dinh2015}, potentially after a random permutation.
However, more structured partitioning has also been explored and is common practice, particularly when modelling images.
For instance, \cite{Dinh2017} used ``masked'' flows that take alternating pixels or blocks of channels in the case of an image in non-volume preserving flows (\textbf{RealNVP}).
In place of permutation \cite{Kingma2018} used $1\times 1$ convolution (\textbf{Glow}).
For the partition for the multi-scale flow in the normalizing direction, \cite{Das2019} suggested selecting features at which the Jacobian of the flow has higher values for the propagated part.

\subsubsection{Autoregressive Flows}
\label{s:autoregressive}
\cite{Kingma2016} used  autoregressive models  as a form of normalizing flow. These are non-linear generalizations of multiplication by a triangular matrix (Section \ref{s:lintriang}).

Let $\couplingfunci(\placeholder \, ; \theta): \R \to \R$ be a bijection parameterized by $\theta$. Then an autoregressive model is a function $\flowfunc: \R^\flowdim \to \R^\flowdim$, which outputs each entry of $\obsval = \flowfunc(\flowval)$ conditioned on the previous entries of the input:
\begin{equation}
\label{e:AF_direct}
\obsvali_t = \couplingfunci(\flowvali_t ; \Theta_t(\flowval_{1:t-1})),
\end{equation}
where $\flowval_{1:t} = (\flowvali_1, \dots, \flowvali_t)$. For $t = 2, \dots, \flowdim$ we choose arbitrary functions $\Theta_t(\placeholder)$ mapping  $\R^{t-1}$ to the set of all parameters, and $\Theta_1$ is a constant. The functions $\Theta_t(\placeholder)$ 
are called \emph{conditioners}.   

The Jacobian matrix of the autoregressive transformation $\flowfunc$ is triangular. Each output $\obsvali_{t}$ only depends on $\flowval_{1:t}$, and so the determinant is just a product of its diagonal entries: 
\begin{equation}
\det \left( \flowfuncJ \right) =
\prod_{t=1}^\flowdim \frac{\partial \obsvali_t}{\partial \flowvali_t} .
\end{equation}
In practice, it's possible to efficiently compute all the entries of the direct flow (Equation \eqref{e:AF_direct}) in one pass using a single network with appropriate masks 
\citep{Germain2015}.  This idea was used by \cite{Papamakarios2017} to create masked autoregressive flows (\textbf{MAF}).

However, the computation of the inverse is more  challenging.  Given the inverse of  $\couplingfunci $, the inverse of $\flowfunc$ can be found with recursion: we have $\flowvali_1 = \couplingfunci^{-1}(\obsvali_1 ; \Theta_1)$ and for any $t = 2, \dots, \flowdim$, $\flowvali_t = \couplingfunci^{-1}(\obsvali_t ; \Theta_t(\flowval_{1:t-1}))$. This computation is inherently sequential which makes it difficult to implement efficiently on modern hardware as it cannot be parallelized.

Note that the functional form for the autoregressive model is very similar to that for the coupling flow. In both cases a bijection $\couplingfunci$ is used, which has as an input one part of the space and  which is parameterized conditioned on the other  part.
We call this bijection a \textit{coupling function} in both cases.
Note that \cite{Huang2018NeuralFlows} used the name ``transformer'' (which has nothing to do with  transformers in NLP).

Alternatively, \citet{Kingma2016} introduced the ``inverse autoregressive flow'' (\textbf{IAF}),  which outputs each entry of  $\obsval$ conditioned the previous entries of  $\obsval$ (with respect to the fixed ordering). Formally,
\begin{equation}
\label{eq:AF_inv}
\obsvali_t = \couplingfunci(\flowvali_t ; \theta_t(\obsval_{1:t-1})).
\end{equation}
One can see that the functional form of the inverse autoregressive flow is the same as the form of the inverse of the flow in Equation \eqref{e:AF_direct}, hence the name.
Computation of the IAF is sequential and expensive, but the inverse of IAF (which is a direct autoregressive flow) can be computed relatively efficiently (Figure \ref{fig:af}). 

\begin{figure}[h]
\centering
\includegraphics[scale=0.5]{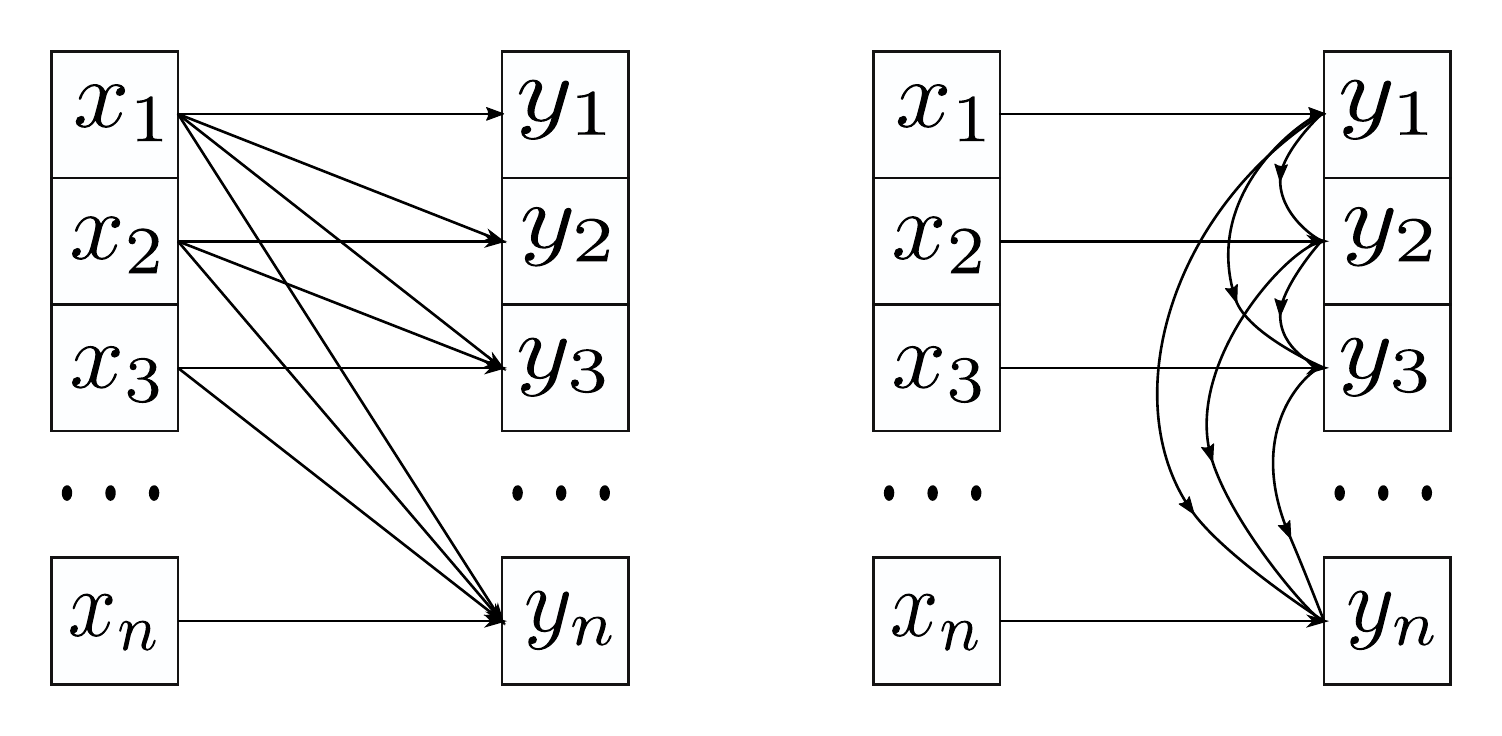}
\caption{Autoregressive flows. On the left, is the direct autoregressive flow given  in Equation \eqref{e:AF_direct}. Each output depends on the current and previous inputs and so this operation can be easily parallelized.  On the right, is the inverse autoregressive flow from Equation \eqref{eq:AF_inv}. Each output depends on the current input and the previous outputs and so computation is inherently sequential and cannot be parallelized.}
\label{fig:af}
\end{figure}

In Section \ref{s:Density_Estimation} we noted that papers typically model flows in the ``normalizing flow'' direction (\ie, in terms of $\flowinv$ from data to the base density) to enable efficient evaluation of the log-likelihood during training.
In this context one can think of IAF as a flow in the generative direction: \ie in terms of $\flowfunc$ from base density to data.
Hence \citet{Papamakarios2017} noted that one should use IAFs if fast sampling is needed (\eg, for stochastic variational inference), and MAFs if fast density estimation is desirable.
The two methods are closely related and, under certain circumstances, are theoretically equivalent  \citep{Papamakarios2017}.  

\subsubsection{Universality}
\label{s:universal}
For several autoregressive flows the universality property has been proven  \citep{Huang2018NeuralFlows, Priyank2019}. Informally, universality means that the flow can learn any target density to any required precision given sufficient capacity and data. 
We will provide a formal proof of the universality theorem following  \cite{Priyank2019}. This section requires some knowledge of measure theory and functional analysis and can be safely skipped.

First,  recall that a mapping $T = (T_1, \dots, T_\flowdim): \R^\flowdim \to \R^\flowdim $ is called triangular if  $T_i$ is a function of $\flowval_{1:i}$ for each $i = 1, \dots, \flowdim$. Such a triangular map $T$ is called increasing if  $T_i$ is an increasing function of $\flowvali_i$ for each $i$. 

\begin{proposition}[\citep{triangular2005}, Lemma 2.1]
\label{p:bogachev}
If $\mu$ and $\nu$ are absolutely continuous Borel probability measures on $\R^\flowdim$, then there exists an increasing triangular transformation $T: \R^\flowdim \to \R^\flowdim$, such that $\nu = T_*\mu$. This transformation is unique up to null sets of $\mu$. A similar result holds for measures on $[0,1]^\flowdim$.
\end{proposition}

\begin{proposition}
If $\mu$ is an absolutely continuous Borel probability measures on $\R^\flowdim$ and $\{T_n\}$ is a sequence of maps $\R^\flowdim \to \R^\flowdim $ which converges pointwise to a map $T$, then a sequence of measures $(T_n)_* \mu$ weakly converges to $T_*\mu$.
\end{proposition}
\begin{proof}
See \cite{Huang2018NeuralFlows}, Lemma 4. The result follows from the dominated convergence theorem. 
\end{proof}

As a corollary, to claim that a class of autoregressive flows $\flowfunc(\placeholder, \theta): \R^\flowdim \to \R^\flowdim $ is universal, it is enough to demonstrate that a family of coupling functions $\couplingfunci$ used in the class is dense in the set of all monotone functions in the pointwise  convergence topology. In particular, \cite{Huang2018NeuralFlows} used neural monotone networks for coupling functions, and \cite{Priyank2019} used monotone polynomials. Using the theory outlined in this section,  universality could also be proved for spline flows \citep{Durkan2019Cubic, Durkan2019Neural} with splines for coupling functions (see Section \ref{s:spline}).

\subsubsection{Coupling Functions}
\label{s:coupling_functions}
As described in the previous sections,  coupling flows and autoregressive flows have a similar functional form and both have coupling functions as  building blocks.
A coupling function is a bijective differentiable function  $ \couplingfunc(\placeholder, \theta): \R^d \to \R^d$, parameterized by $\theta$. 
In coupling flows, these functions are typically constructed by applying a scalar coupling function $\couplingfunci (\placeholder, \theta): \R \to \R$ elementwise.
In autoregressive flows, $d=1$ and hence they are also scalar valued.
Note that scalar coupling functions are necessarily (strictly) monotone. 
In this section we describe the scalar coupling functions commonly used in the literature.

\paragraph{Affine coupling}
  
\label{s:affine}
Two simple forms of coupling functions $ \couplingfunci: \R \to \R$ were proposed by \citep{Dinh2015} in \textbf{NICE} (nonlinear independent component estimation). These were the \emph{additive coupling function}: 
\begin{equation}
\couplingfunci(x \, ; \theta) = x + \theta,
\quad \theta \in \R,
\end{equation}
and the \emph{affine coupling function}: 
\begin{equation}
\couplingfunci(x ; \theta) = \theta_1 x + \theta_2, \quad \theta_1 \neq 0, \ \theta_2 \in \R.
\end{equation}
Affine coupling functions are used for coupling flows in NICE \citep{Dinh2015}, RealNVP \citep{Dinh2017}, Glow  \citep{Kingma2018} and for autoregressive architectures in IAF \citep{Kingma2016} and MAF \citep{Papamakarios2017}.
They are simple and computation is efficient.
However, they are limited in expressiveness and many flows must be stacked to represent complicated distributions. \\

\paragraph{Nonlinear squared flow}  \cite{Ziegler2019} proposed an invertible non-linear squared transformation defined by:
\begin{equation}
    \couplingfunci(x \, ; \theta) = a x + b + \frac{c}{1 + (d x + h)^2} .
\end{equation} 
Under some constraints on parameters $\theta = [a,b,c,d,h] \in \R^5$, the coupling function is invertible and its inverse is analytically computable as a root of a cubic polynomial (with only one real root).
Experiments showed that these coupling functions facilitate learning multimodal distributions. \\

\paragraph{Continuous mixture  CDFs}
\label{s:CDF}
\cite{ho2019flow} proposed the \textbf{Flow++} model, which contained several improvements, including a more expressive coupling function.
The layer is almost like a linear transformation, but one also applies a monotone function to $x$:
\begin{equation}
 \couplingfunci(x ; \theta) = \theta_1  F(x, \theta_3) + \theta_2 ,
\end{equation}
where $\theta_1 \neq 0$, $\theta_2 \in \R$ and $\theta_3 = [\bm{\pi}, \bm{\mu}, \mathbf{s}] \in \R^K\times \R^K \times \R^K_+ $.
The function $F(x, \bm{\pi}, \bm{\mu}, \mathbf{s}) $ is the CDF of a mixture of $K$ logistics, postcomposed with an inverse sigmoid:
\begin{equation}
 F(x, \bm{\pi}, \bm{\mu}, \mathbf{s}) = \sigma^{-1}\left(\sum_{j=1}^K \pi_j \sigma\left(\frac{x-\mu_j}{s_j}\right) \right) .
\end{equation}
Note, that the post-composition with $\sigma^{-1}: [0,1] \to \R$ is used to ensure the right range for $\couplingfunci$.
Computation of the inverse is done numerically with the bisection algorithm.
The derivative of the transformation with respect to $x$ is expressed in terms of PDF of logistic mixture (\ie, a linear combination of hyperbolic secant functions), and its computation is not expensive.
An ablation study demonstrated that switching from an affine coupling function to a logistic mixture improved performance slightly. \\

\paragraph{Splines}
\label{s:spline}
A spline is a piecewise-polynomial or a piecewise-rational function which  is specified by $K + 1$ points $(\flowvali_i,\obsvali_i)_{i=0}^K$, called \textit{knots}, through which the spline passes.
To make a useful coupling function, the spline should be monotone which will be the case if 
 $\flowvali_i < \flowvali_{i+1}$ and  $\obsvali_i < \obsvali_{i+1}$.  Usually splines are considered on a compact interval.

\textbf{Piecewise-linear and piecewise-quadratic:}

\cite{Muller2018NeuralSampling} used linear splines for coupling functions $\couplingfunci: [0,1] \to [0,1]$. They 
divided the domain into $K$ equal bins. Instead of defining increasing values for $\obsvali_i$, they modeled $\couplingfunci$ as the integral of a positive piecewise-constant function:
\begin{equation}
\couplingfunci(x ; \theta) = \alpha \theta_b + \sum_{k=1}^{b-1}\theta_k, 
\end{equation}
where $\theta \in \R^K$ is a probability vector, 
$b = \lfloor K x \rfloor$ (the bin that contains $x$), and $\alpha = Kx - b$  (the position of $x$ in bin $b$). This map is invertible, if all $\theta_k > 0$, with  derivative:   
$
\frac{\partial \couplingfunci}{\partial x} = \theta_bK.$

\cite{Muller2018NeuralSampling} also used a monotone quadratic spline on the unit interval for a coupling function and modeled this as the integral of a positive piecewise-linear function.  A monotone quadratic spline is invertible; finding its inverse map requires solving a quadratic equation. 

\textbf{Cubic Splines:}
\cite{Durkan2019Cubic} proposed using monotone cubic splines for a coupling function.
They do not restrict the domain to the unit interval, but instead use the form:
$\couplingfunci(\placeholder ; \theta) = \sigma^{-1} ( \hat{\couplingfunci}(\sigma( \placeholder ) ; \theta) )$, where $\hat{\couplingfunci}(\placeholder ; \theta): [0,1] \to [0,1]$ is a monotone cubic spline and $\sigma$ is a sigmoid.  Steffen's method is used to construct the spline. Here, one specifies $K + 1$ knots of the spline and boundary derivatives $\hat{\couplingfunci}'(0) $ and $\hat{\couplingfunci}'(1)$. These quantities are modelled as the output of a neural network.

Computation of the derivative is easy as it is piecewise-quadratic. A monotone cubic polynomial has only one real root and for inversion, one can find this either analytically or numerically. 
However, the procedure is numerically unstable if not treated carefully.
The flow can be trained by gradient descent by differentiating through the numerical root finding method.
However, \cite{Durkan2019Neural}, noted numerical difficulties when the sigmoid saturates for values far from zero.

\textbf{Rational quadratic splines:}
\cite{Durkan2019Neural}  model a coupling function $\couplingfunci(x \, ; \theta)$ as a monotone rational-quadratic spline on an interval as the identity function otherwise.
They define the spline using the method of \citet{Gregory1982}, by specifying  $K + 1$ knots $\{ \couplingfunci(\flowvali_i) \}_{i=0}^{K}$ and the  derivatives at the inner points: $\{ \couplingfunci'(\flowvali_i) \}_{i=1}^{K-1}$. These locations of the knots and their derivatives are modelled as the output of a neural network.

 The derivative with respect to $x$ is a quotient derivative and the function can be inverted by solving a quadratic equation.  \cite{Durkan2019Neural} used this coupling function with both a coupling architecture \textbf{RQ-NSF(C)} and an auto-regressive architecture \textbf{RQ-NSF(AR)}. \\

\paragraph{Neural autoregressive flow}
\label{NAF}
\cite{Huang2018NeuralFlows} introduced Neural Autoregressive Flows (\textbf{NAF}) where  a coupling function $\couplingfunci(\placeholder \, ; \theta) $ is modelled with a deep neural network. Typically such a network is not invertible, but they proved a  sufficient condition for it to be bijective:
\begin{proposition}
\label{p:monotone_nn}
If $\NN(\placeholder): \R \to \R$ is a multilayer percepton, such that all weights are positive and all activation functions are strictly monotone, then $\NN(\placeholder)$ is a strictly monotone function.
\end{proposition}

They proposed two forms of neural networks: the deep sigmoidal coupling function (\textbf{NAF-DSF}) and deep dense sigmoidal coupling function (\textbf{NAF-DDSF}).
Both are  MLPs with layers of sigmoid and logit units and non-negative weights; the former has a single hidden layer of sigmoid units, whereas the latter is more general and does not have this bottleneck. By Proposition~\ref{p:monotone_nn}, the resulting $\couplingfunci(\placeholder \, ; \theta)$ is a strictly monotone function. They also proved that a DSF network can approximate any strictly monotone univariate function and so NAF-DSF is a universal flow.

\cite{Wehenkel2019} noted that imposing positivity of weights on a flow makes training harder and requires more complex conditioners.
To mitigate this,  they introduced unconstrained monotonic neural networks (\textbf{UMNN}).
The idea is in order to model a strictly monotone function, one can describe a strictly positive (or negative) function with a neural network and then integrate it numerically. They demonstrated that UMNN requires less parameters than NAF to reach similar performance, and so is more scalable for high-dimensional datasets. \\

\paragraph{Sum-of-Squares polynomial flow}
\label{s:sos}
\cite{Priyank2019} modeled $\couplingfunci(\placeholder \, ; \theta) $ as a strictly increasing polynomial. They proved such polynomials can approximate any strictly monotonic univariate continuous function. Hence, the resulting flow (\textbf{SOS} - sum of squares polynomial flow) is a universal flow.

The authors observed that the derivative of an increasing single-variable polynomial is a positive polynomial. Then they used a classical result from algebra:  all positive  single-variable polynomials are the sum of squares of polynomials.
To get the coupling function, one needs to integrate the sum of squares:
\begin{equation}
\couplingfunci(x \, ; \theta)  = c + \int_0^x \sum_{k=1}^K\left(\sum_{l=0}^L a_{kl} u^l \right)^2 du \ ,
\end{equation}
where $L$ and $K$ are hyperparameters (and, as noted in the paper, can be chosen to be $2$).

SOS is easier to train than NAF, because there are no restrictions on the parameters (like positivity of weights). For L=0,
 SOS reduces to the affine coupling function and so it is a generalization of the basic affine  flow. \\

\paragraph{Piecewise-bijective coupling}
\label{piecewise-bijective}

\cite{dinh2019rad} explore the idea that a coupling function  does not need to be bijective, but just piecewise-bijective (Figure \ref{f:rad}).
Formally, they consider a function $\couplingfunci(\placeholder \, ; \theta): \R \to \R $ and a covering of the domain into $K$ disjoint subsets: $\R = \bigsqcup_{i=1}^K A_i$, such that the restriction of the function onto each subset $\couplingfunci(\placeholder \, ; \theta)|_{A_i} $ is injective.  

\begin{figure}[h]
\centering
\includegraphics[scale=0.46]{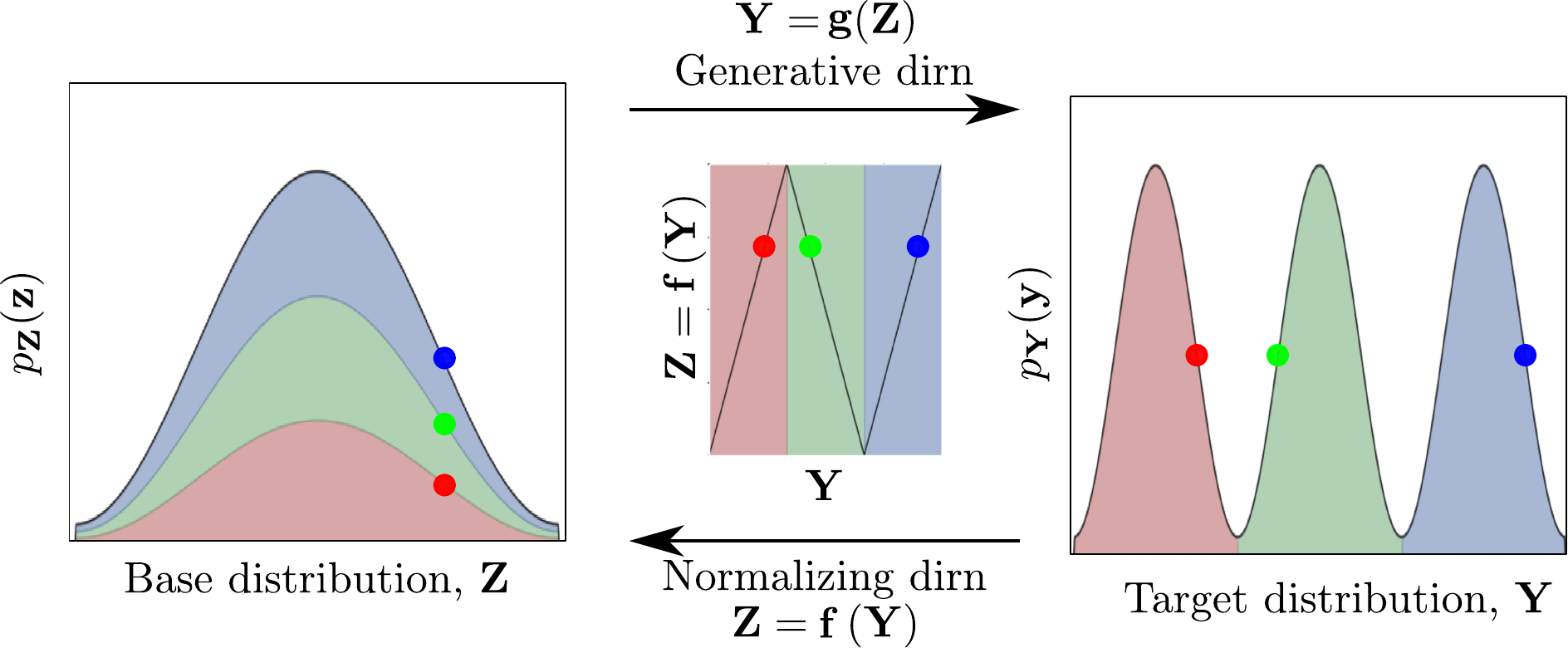}
\caption{Piecewise bijective coupling.  The target domain (right) is divided into disjoint sections (colors) and each mapped by a monotone function (center) to the base distribution (left).  For inverting the function, one samples a component of the base distribution using a gating network.}
\label{f:rad}
\end{figure}

\cite{dinh2019rad} constructed a flow  $\flowinv: \R^\flowdim \to \R^\flowdim$ with a coupling architecture and piecewise-bijective coupling function in the normalizing direction - from data distribution to (simpler) base distribution.  There is a covering of the data domain, and each subset of this covering is separately mapped to the base distribution.  Each part of the base distribution now receives contributions from each subset of the data domain.  For sampling, \cite{dinh2019rad} proposed a probabilistic mapping from the base to data domain.

More formally, denote the target $\obsvali$ and base $\basevali$, and consider a  lookup function $\phi: \R \to [K] = \{1, \dots, K\}$, such that $\phi(y) = k$, if $y \in A_k$. One can define a new map  $\R \to \R \times [K] $, given by the rule $y \mapsto (\couplingfunci(y), \phi(y))$, and a density on a target space $p_{\basevari,[K]}(\basevali, k) =    p_{[K]|\basevari}(k|\basevali) p_\basevari(\basevali)$.
One can think of this as an unfolding of the non-injective map $\couplingfunci$.
In particular, for each point $\basevali$ one can find its pre-image by sampling from $p_{[K]|\basevari}$, which is called a \emph{gating network}.
Pushing forward along this unfolded map is now well-defined and one gets the formula for the density $p_{\obsvari}$: 
\begin{equation}
     p_{\obsvari}(\obsvali) = p_{\basevari,[K]}(\couplingfunci(\obsvali),\phi(\obsvali)) |\jac{\couplingfunci} (\obsvali)| .
\end{equation}
This  real and discrete (\textbf{RAD}) flow efficiently learns distributions with discrete structures (multimodal distributions,  distributions with holes, discrete symmetries etc).

\subsection{Residual Flows}
\label{s:residual}
Residual networks \citep{He2016} are compositions of the function of the form 
\begin{equation}
\label{resnet}
 \flowfunc(\flowval) = \flowval + F(\flowval).  
\end{equation}
Such a function is called \emph{a residual connection}, and here the \emph{residual block} $F(\placeholder)$ is a feed-forward neural network of any kind (a CNN in the original paper).

The first attempts to build a reversible network architecture based on residual connections were made in \textbf{RevNets} \citep{Gomez2017} and \textbf{iRevNets} \citep{jacobsen2018irevnet}. Their main motivation was to save memory during training and to stabilize computation.  The central idea is a variation of additive coupling functions: consider a disjoint partition of $\R^\flowdim = \R^d \times \R^{\flowdim-d}$ denoted by $\flowval = (\flowval^A,\flowval^B)$ for the input and $\obsval = (\obsval^A,\obsval^B)$ for the output, and define a function:
\begin{align}
    \obsval^A = \flowval^A + F(\flowval^B) \nonumber\\
\obsval^B = \flowval^B + G(\obsval^A),
\end{align}
where $F: \R^{\flowdim -d} \to \R^d$ and $ G: \R^d \to \R^{\flowdim -d}$ are residual blocks. This network is invertible (by re-arranging the equations in terms of $\flowval^A$ and $\flowval^B$ and reversing their order) but computation of the Jacobian is inefficient.

A different point of view on reversible networks comes from a dynamical systems perspective via the observation that a residual connection is a discretization of a first order ordinary differential equation (see Section \ref{s:infinitesimal} for more details).  \cite{Chang2018, chang2018antisymmetricrnn} proposed several architectures, some of these networks were demonstrated to be invertible. However, the Jacobian determinants of these networks cannot be computed efficiently.

Other research has focused on making the residual connection $\flowfunc(\placeholder)$ invertible. A sufficient condition for the invertibility was found in  \citep{Behrmann}. They proved the following statement:
\begin{proposition}
A residual connection \eqref{resnet} is invertible, if the Lipschitz constant of the residual block is $\text{Lip}(F) < 1$.
\end{proposition}
There is no analytically closed form for the inverse, but it can be found numerically using fixed-point iterations (which, by the Banach theorem, converge if we assume  $\text{Lip}(F) < 1$).

Controlling the Lipschitz constant of a neural network is not simple.  The specific architecture proposed by  \cite{Behrmann}, called \textbf{iResNet}, uses a convolutional network for the residual block.
It constrains the spectral radius of each convolutional layer in this network to be less than one.

The Jacobian determinant of the iResNet cannot be computed directly, so the authors propose to use a (biased) stochastic estimate.  The Jacobian of the residual connection $\flowfunc$ in Equation \eqref{resnet} is: $\flowfuncJ = I + \jac{F} $. Because the function $F$ is assumed to be Lipschitz with $\text{Lip}(F) < 1$, one has: $ \left| \det (I + \jac{F}) \right| = \det (I + \jac{F})$.  Using the linear algebra identity, $\ln{\det \mathbf{A}} = \text{Tr}\ln{\mathbf{A}} $ we have:

\begin{equation}
   \ln{\left| \det \flowfuncJ \right| } =
   \ln{ \det (I + \jac{F})} = \text{Tr}( \ln {(I + \jac{F})}),
\end{equation}
 
\noindent Then one considers a power series for the trace of the matrix logarithm:
\begin{equation}
\label{logseries}
    \text{Tr}(\ln{ (I + \jac{F})}) = \sum_{k=1}^\infty (-1)^{k+1} \frac{\text{Tr} (\jac{F})^k}{k}.
\end{equation}
By truncating this series one can calculate an approximation to the log Jacobian determinant of $\flowfunc$. To efficiently compute each member of the truncated series, the Hutchinson trick was used.  This trick provides a stochastic estimation of of a matrix trace $\mathbf{A} \in \R^{\flowdim \times \flowdim}$, using the relation: $\text{Tr} \mathbf{A} = \mathbb{E}_{p(\mathbf{v})} [\mathbf{v}^T\mathbf{A}\mathbf{v}]$, where $\mathbf{v} \in \R^\flowdim$,  $\mathbb{E}[\mathbf{v}] = 0$, and $\text{cov}(\mathbf{v}) = I$. 

 Truncating the power series gives a biased estimate of the log Jacobian determinant (the bias depends on the  truncation error).
 An unbiased stochastic estimator was proposed by \cite{iresnetunbiased} in a model they called a \textbf{Residual flow}.
 The authors used a \emph{Russian roulette} estimator instead of truncation. 
Informally, every time one adds the next term $a_{n+1}$ to the partial sum $\sum_{i=1}^n a_i $ while calculating the series $\sum_{i=1}^\infty a_i$, one flips a coin to decide if the calculation should be continued or stopped.
During this process one needs to  re-weight terms for an unbiased estimate.

\subsection{Infinitesimal (Continuous) Flows}\label{s:infinitesimal}

The residual connections discussed in the previous section can be viewed as discretizations of a first order ordinary differential equation (ODE) \citep{haber2017, e2017}:
\begin{equation}
\label{ode}
 \frac{d}{dt}\flowval(t) = F(\flowval(t), \theta(t)), 
\end{equation}
where $F: \R^\flowdim \times \Theta \to \R^\flowdim$ is a function which determines the dynamic (the \emph{evolution function}), $\Theta$ is a set of parameters and $\theta: \R \to \Theta$ is a parameterization. The discretization of this equation (Euler's method) is 
\begin{equation}
 \flowval_{n+1} - \flowval_{n} = \varepsilon F(\flowval_n, \theta_n)  , 
\end{equation}
and this is equivalent to a residual connection with a residual block $\varepsilon F(\placeholder, \theta_n)$.

In this section we consider the case where we do not discretize but try to learn the continuous dynamical system instead. Such flows are called  \textit{infinitesimal} or \textit{continuous}.  We consider two distinct types. The formulation of the first type comes from ordinary differential equations, and of the second type from stochastic differential equations.

\subsubsection{ODE-based methods}
\label{s:ODE}

Consider an ODE as in Equation \eqref{ode}, where $t \in [0,1]$.  Assuming  uniform Lipschitz continuity in $\flowval$ and continuity in $t$, the solution exists (at least, locally) and, given an initial condition $\flowval(0) = \baseval $, is unique (Picard-Lindel{\"o}f-Lipschitz-Cauchy theorem \citep{ArnoldODE}). 
We denote the solution at each time $t$ as $\Phi^t(\baseval)$.

\begin{remark}
At each time $t$, $ \Phi^t(\placeholder): \R^\flowdim \to \R^\flowdim $ is a diffeomorphism and satisfies the group law: $\Phi^t \circ \Phi^s = \Phi^{t+s}$. Mathematically speaking, an ODE \eqref{ode} defines a one-parameter group of diffeomorphisms on $\R^\flowdim$. Such a group is called a smooth flow in dynamical systems theory and differential geometry \citep{intro_to_dynamical_systems}.
\end{remark}

When $t=1$, the diffeomorphism $\Phi^1(\placeholder) $ is called a \emph{time one map}. The idea to model a normalizing flow as a time one map  $\obsval =\flowfunc(\baseval) = \Phi^1(\baseval)$ was presented by \citep{chen2018neural} under the name \textbf{Neural ODE (NODE)}. 
From a deep learning perspective this can be seen as an ``infinitely deep'' neural network with input $\baseval$, output $\obsval$ and continuous weights $\theta(t)$.
The invertibility of such networks naturally comes from the theorem of the existence and uniqueness of the solution of the ODE.  

Training these networks for a supervised downstream task can be done by  the \emph{adjoint sensitivity method} which is the continuous analog of backpropagation. It computes the gradients of the loss function by solving a second (\emph{augmented}) ODE backwards in time.  For loss $L(\flowval(t))$, where $\flowval(t)$ is a solution of ODE \eqref{ode}, its sensitivity or adjoint is $\mathbf{a}(t) = \frac{dL}{d\flowval(t)}$. This is the analog of the derivative of the loss with respect to the hidden layer.
In a standard neural network, the backpropagation formula computes this derivative: $\frac{dL}{d\mathbf{h}_{n}} = \frac{dL}{d\mathbf{h}_{n+1}}\frac{d\mathbf{h}_{n+1}}{d\mathbf{h}_{n}}$. For ``infinitely deep'' neural network, this formula changes into an ODE: 

\begin{equation}\frac{d \mathbf{a}(t)}{d t} = -\mathbf{a}(t)\frac{d F(\flowval(t), \theta(t))}{d \flowval(t) }.
\end{equation}

For density estimation learning, we do not have a loss, but instead seek to maximize the log likelihood.  For normalizing flows, the change of variables formula is given by another ODE:
\begin{align}
\label{e:Pontryagin}
\frac{d}{d t} \log(p(\flowval(t)))
= -\tr\left( \frac{d F(\flowval(t))}{d\flowval(t)} \right).
\end{align}
Note that we no longer need to compute the determinant.  To train the model and sample from $\obsdensity$ we solve these ODEs, which can be done with any numerical ODE solver.

\citet{GrathwohlFFJORD:MODELS} used the Hutchinson estimator to calculate an unbiased stochastic estimate of the trace-term.  This approach which they termed \textbf{FFJORD} reduces the complexity even further. \cite{Finlay2020trainNODE} added two regularization terms into the loss function of FFJORD: the first term forces solution trajectories to follow straight lines with constant speed, and the second term is the Frobenius norm of the Jacobian. This regularization decreased the training time significantly and reduced the need for multiple GPUs.
An interesting side-effect of using continuous ODE-type flows is that one needs fewer parameters to achieve the similar performance.  For example, \citet{GrathwohlFFJORD:MODELS} show that for the comparable performance on CIFAR10, FFJORD uses less than $2\% $ as many parameters as Glow.

Not all diffeomorphisms can be presented as a time one map of an ODE (see \citep{intro_to_dynamical_systems, time_one_maps}).  For example, one necessary condition is that the map is {\em orientation preserving} which means that the Jacobian determinant must be positive.  This can be seen because the solution $\Phi^t$ is a (continuous) path in the space of diffeomorphisms from the identity map $\Phi^0 = Id $ to the time one map $\Phi^1$. Since the Jacobian determinant of a diffeomorphism is nonzero, its sign cannot change along the path. Hence, a time one map must have a positive Jacobian determinant.  For example, consider a map $f: \R \to \R$, such that $f(x) = -x$. It is obviously a diffeomorphism, but it can not be presented as a time one map of any ODE, because it is not orientation preserving.

 \cite{augmentedNeuralODE} suggested how one can improve Neural ODE in order to be able to represent a broader class of diffeomorphisms. Their model is called Augmented Neural ODE (\textbf{ANODE}). They add variables $\mathbf{\hat{x}}(t) \in \R^p$ and consider a new ODE:
 \begin{equation}
 \frac{d}{dt}\begin{bmatrix}\flowval(t) \\ \mathbf{\hat{x}}(t) \end{bmatrix} = \hat{F}\left(\begin{bmatrix}\flowval(t) \\ \mathbf{\hat{x}}(t) \end{bmatrix}, \theta(t) \right)   
 \end{equation}
with initial conditions $\flowval(0)=\baseval$ and $\mathbf{\hat{x}}(0)=0$. The addition of $\mathbf{\hat{x}}(t)$ in particular gives freedom for the Jacobian determinant to remain positive. As was demonstrated in the experiments, ANODE is capable of learning distributions that the Neural ODE cannot, and the training time is shorter. 
\cite{Zhang2019ode} proved that any diffeomorphism can be represented as a time one map of ANODE and so this is a universal flow.

A similar ODE-base approach was taken by \cite{Salman2018DeepDN} in Deep Diffeomorphic Flows. In addition to  modelling a path $ \Phi^t(\placeholder)$  in the space of all diffeomorphic transformations, for $t \in [0,1]$, they proposed geodesic regularisation in which longer paths are punished.

\subsubsection{SDE-based methods (Langevin flows)}

The idea of the Langevin flow is simple; we start with a complicated and irregular data distribution $\obsdensity(\obsval)$ on $\R^\flowdim$, and then mix it to produce the simple base distribution $\basedensity(\baseval)$. If this mixing obeys certain rules, then this procedure can be invertible. This idea was explored by \citet{Fokker-Plank-machine, Welling2011BayesianLV, Rezende2015, Sohl-DicksteinW15, Salimans2015, Jankowiak2018, Chen2018CTF}. We provide a high-level overview of the method, including the necessary mathematical background.

A stochastic differential equation (SDE) or It{\^o} process describes a change of a random variable $\flowval\in \R^\flowdim$ as a function of time $t \in \R_+$:
\begin{align}
\label{SDE}
    d \flowval (t) = b(\flowval (t), t)dt  +  \sigma(\flowval (t), t)dB_t,
\end{align}

\noindent where $b(\flowval, t) \in \R^\flowdim$ is the \textit{drift coefficient}, $\sigma(\flowval, t) \in \R^{\flowdim\times\flowdim}$ is the \textit{diffusion coefficient}, and $B_t$ is $\flowdim$-dimensional \textit{Brownian motion}.  
One can interpret the drift term as a deterministic change and the diffusion term as providing the stochasticity and mixing. Given some assumptions about these functions, the solution exists and is unique \citep{Oksendal}.

Given a time-dependent random variable $\flowval (t)$ we can consider its  density function $p(\flowval, t)$ and this is also time dependent. If $\flowval (t)$ is a solution of Equation \eqref{SDE}, its density function satisfies two partial differential equations describing the  forward and backward evolution \citep{Oksendal}.  The forward evolution is given by Fokker-Plank equation or Kolmogorov's forward equation:

\begin{equation}
   \resizebox{0.5\textwidth}{!}{$ \frac{\partial}{\partial t} p(\flowval, t) = - \nabla_\flowval \cdot (b(\flowval,t)p(\flowval, t)) + \sum_{i,j} \frac{\partial^2}{\partial \flowvali_i \partial \flowvali_j}D_{ij}(\flowval, t)p(\flowval, t),$}
\end{equation}

\noindent where $D = \frac{1}{2}\sigma \sigma^T$, with the initial condition $p(\placeholder, 0) = \obsdensity(\placeholder)$. The reverse is given by Kolmogorov's backward equation: 
\begin{equation}
 \resizebox{0.5\textwidth}{!}{$ - \frac{\partial}{\partial t} p(\flowval, t) = b(\flowval,t) \cdot \nabla_\flowval (p(\flowval, t)) + \sum_{i,j} D_{ij}(\flowval, t) \frac{\partial^2}{\partial \flowvali_i \partial \flowvali_j}p(\flowval, t), $}
\end{equation}

\noindent where $0<t<T$, and the initial condition is $p(\placeholder, T) = \basedensity(\placeholder)$.

Asymptotically the Langevin flow can learn any distribution if one picks the drift and diffusion coefficients appropriately \citep{Fokker-Plank-machine}. However this result is not very practical, because one needs to know the (unnormalized) density function of the data distribution.

One can see that if the diffusion coefficient is zero, the It{\^o} process reduces to the ODE \eqref{ode}, and the Fokker-Plank equation becomes a Liouville's equation, which is connected to Equation \eqref{e:Pontryagin} (see \cite{chen2018neural}). It is also equivalent to the form of the transport equation considered in  \cite{Jankowiak2018} for stochastic optimization.

 \cite{Sohl-DicksteinW15} and \cite{Salimans2015}  suggested using MCMC methods to model the diffusion. They considered discrete time $t = 0, \dots, T$. For each time $t$, 
 $\flowval^t$ is a random variable where $\flowval^0 = \obsval $ is the data point, and $\flowval^T = \baseval $ is the base point. The forward transition probability $q(\flowval^t | \flowval^{t-1} ) $ is taken to be either normal or binomial distribution with trainable parameters. 
  Kolmogorov's backward equation implies that the backward transition $p(\flowval^{t-1}| \flowval^{t}) $ must have the same functional form as the forward transition (\ie, be either normal or binomial). Denote: $q(\flowval^0) = \obsdensity(\obsval)$, the data distribution,  and $p(\flowval^T) = \basedensity(\baseval)$, the base distribution. Applying the backward transition to the base distribution, one obtains a new density $p(\flowval^0)$, which one wants to match with  $q(\flowval^0)$. Hence, the optimization objective is the log likelihood  $L = \int d\flowval^0 q(\flowval^0) \log p(\flowval^0) $. This is intractable, but one can find a lower bound as in variational inference. 

Several papers have worked explicitly  with the SDE \citep{Chen2018CTF, NSDE1, NSDE2, Liutkus2018SlicedWassersteinFN, Li2020SDE}. \cite{Chen2018CTF} use SDEs to create an interesting posterior for variational inference.  They sample a latent variable $\baseval_0$ conditioned on the input $\flowval$, and then evolve $\baseval_0$ with SDE. In practice this evolution is computed by discretization. By analogy to Neural ODEs, \textbf{Neural Stochastic Differential Equations}   were proposed \citep{NSDE1, NSDE2}. In this approach coefficients of the SDE are modelled as neural networks, and black box SDE solvers are used for inference.    To train Neural SDE one needs an analog of backpropagation, \cite{NSDE2} proposed the use of Kunita's theory of stochastic flows. Following this, \cite{Li2020SDE}  derived the adjoint SDE whose solution gives the gradient of the original Neural SDE.                 
 
 Note, that even though Langevin flows manifest nice mathematical properties, they have not found practical applications. In particular, none of the methods has  been tested on baseline datasets for flows.

\section{Datasets and performance}
\label{performance}

In  this section we discuss  datasets commonly used for training and testing normalizing flows. We provide  comparison tables of the results as they were presented in the corresponding papers. The list of the flows for which we post the performance results is given in Table \ref{table:FlowsSummary}.

\begin{table}[!t]
\centering
\caption{List of Normalizing Flows for which we show performance results.}
\label{table:FlowsSummary}
\begin{tabular}{ ||l|l|l|| }
 \hline
 Architecture & Coupling function & Flow name\\ [0.5ex] 
 \hline\hline
\multirow{5}{*}{Coupling, \ref{coupling}} &\
\multirow{2}{*}{Affine, \ref{s:affine}} &  RealNVP \\& & Glow \\ \cline{2-3}
 & Mixture CDF, \ref{s:CDF} & Flow++ \\ \cline{2-3}
 & \multirow{2.5}{*}{Splines, \ref{s:spline}} & quadratic (C) \\ & & cubic \\ & & RQ-NSF(C) \\  \cline{2-3}
 & Piecewise Bijective, \ref{piecewise-bijective} & RAD \\ \hline
\multirow{4}{*}{Autoregressive, \ref{s:autoregressive}} & Affine & MAF \\  \cline{2-3}
 & Polynomial, \ref{s:sos} & SOS \\ \cline{2-3}
 & \multirow{2}{*}{Neural Network, \ref{NAF} } & NAF \\ & & UMNN \\
 \cline{2-3} & \multirow{2}{*}{Splines} & quadratic (AR) \\ & & RQ-NSF(AR) \\ \hline
\multirow{2}{*}{Residual, \ref{s:residual}} & \multirow{2}{*}{ } & iResNet \\ \cline{3-3} & &  Residual flow\\  \hline
\multirow{1}{*}{ODE, \ref{s:ODE}} & \multirow{2}{*}{ } & FFJORD  \\  \hline
\end{tabular}
\end{table}

\subsection{Tabular datasets}

We describe datasets as they were preprocessed in \cite{Papamakarios2017} (Table \ref{table:TabularDims})\footnote{See https://github.com/gpapamak/maf}. These datasets are relatively small and so are a reasonable first test of unconditional density estimation models. All datasets were cleaned and de-quantized by adding uniform noise, so they can be considered samples from an absolutely continuous distribution. 

We use a collection of datasets from the UC Irvine machine learning repository \citep{uci}.
\begin{enumerate}
    \item POWER: a collection of electric power consumption measurements in one house over 47 months. 
    \item GAS: a collection of measurements from chemical sensors in several gas mixtures.
    \item HEPMASS: measurements from high-energy physics experiments aiming to detect particles with unknown mass.
     \item MINIBOONE: measurements from MiniBooNE experiment for observing  neutrino oscillations.
\end{enumerate}

In addition we consider the Berkeley segmentation dataset \citep{bsds} which contains segmentations of natural images. \cite{Papamakarios2017} extracted $8\times 8$ random monochrome patches from it.

\begin{table*}[!t]
\caption{Tabular datasets: data dimensionality and number of training examples.}
\label{table:TabularDims}
\centering
 \begin{tabular}{||c c c c c c||} 
 \hline
  & POWER & GAS & HEPMASS & MINIBOONE & BSDS300 \\ [0.5ex] 
 \hline\hline
 Dims & 6 & 8 & 21 & 43 & 63 \\ 
 \hline
 \#Train & $\approx 1.7$M & $\approx 800$K & $\approx 300$K & $\approx 30$K & $\approx 1$M\\ [1ex] 
 \hline
\end{tabular}
\end{table*}

In Table \ref{table:CompTabular} we compare performance of flows for these tabular datasets.  For experimental details, see the following papers:  RealNVP \citep{Dinh2017} and MAF \citep{Papamakarios2017}, Glow \citep{Kingma2018} and FFJORD
 \citep{GrathwohlFFJORD:MODELS}, NAF \citep{Huang2018NeuralFlows}, UMNN \citep{Wehenkel2019}, SOS \citep{Priyank2019}, Quadratic Spline flow and RQ-NSF \citep{Durkan2019Neural}, Cubic Spline Flow \citep{Durkan2019Cubic}. 
 
Table \ref{table:CompTabular} shows that universal flows (NAF, SOS, Splines) demonstrate relatively better performance.

\begin{table*}[!t]
\caption{Average test log-likelihood (in nats) for density estimation on tabular  datasets (higher the better). A number in parenthesis next to a flow indicates number of layers. MAF MoG is MAF with mixture of Gaussians as a base density.  }
\label{table:CompTabular}
\centering
 \begin{tabular}{||c c c c c c||} 
 \hline
  & \small{POWER} & \small{GAS} & \small{HEPMASS} & \small{MINIBOONE} & \small{BSDS300} \\ [0.5ex] 
 \hline\hline
 \small{MAF(5)} & 0.14\tiny{$\pm 0.01$} & 9.07\tiny{$\pm 0.02$} & -17.70\tiny{$\pm 0.02$} & -11.75\tiny{$\pm 0.44$}  & 155.69\tiny{$\pm 0.28$} 
 \\ 
 \hline
 \small{MAF(10)} & 0.24\tiny{$\pm 0.01$} & 10.08\tiny{$\pm 0.02$}  & -17.73\tiny{$\pm 0.02$} & -12.24\tiny{$\pm 0.45$} & 154.93\tiny{$\pm 0.28$} 
 \\ 
 \hline
 \small{MAF MoG} & 0.30\tiny{$\pm 0.01$} & 9.59\tiny{$\pm 0.02$} & -17.39\tiny{$\pm 0.02$}  & -11.68\tiny{$\pm 0.44$}  & 156.36\tiny{$\pm 0.28$} 
 \\ 
 \hline
 \small{RealNVP(5)} & -0.02\tiny{$\pm 0.01$} & 4.78\tiny{$\pm 1.8$} & -19.62\tiny{$\pm 0.02$} & -13.55\tiny{$\pm 0.49$} & 152.97\tiny{$\pm 0.28$}
 \\ 
 \hline
 \small{RealNVP(10)} & 0.17\tiny{$\pm 0.01$}  & 8.33\tiny{$\pm 0.14$} & -18.71\tiny{$\pm 0.02$}  & -13.84\tiny{$\pm 0.52$} & 153.28\tiny{$\pm 1.78$}
 \\ 
 \hline
  \small{Glow} & 0.17  & 8.15 & -18.92 & -11.35 & 155.07
 \\ 
 \hline
  \small{FFJORD} & 0.46 & 8.59  & -14.92 & -10.43 &
  157.40
 \\ 
 \hline
 \small{NAF(5)} & 0.62\tiny{$\pm 0.01$}  & 11.91\tiny{$\pm 0.13$} & -15.09\tiny{$\pm 0.40$}  & \textbf{-8.86}\tiny{$\pm 0.15$} & 157.73\tiny{$\pm 0.04$}
 \\ 
 \hline
  \small{NAF(10)} & 0.60\tiny{$\pm 0.02$}  & 11.96\tiny{$\pm 0.33$} & -15.32\tiny{$\pm 0.23$} & -9.01\tiny{$\pm 0.01$} & 157.43\tiny{$\pm 0.30$}
 \\ 
  \hline
  \small{UMNN} & 0.63\tiny{$\pm 0.01$}  & 10.89\tiny{$\pm 0.70$} & \textbf{-13.99}\tiny{$\pm 0.21$} & -9.67\tiny{$\pm 0.13$}  & \textbf{157.98}\tiny{$\pm 0.01$}
 \\ 
  \hline
  \small{SOS(7)} & 0.60\tiny{$\pm 0.01$} & 11.99\tiny{$\pm 0.41$} & -15.15\tiny{$\pm 0.10$} & -8.90\tiny{$\pm 0.11$}  & 157.48\tiny{$\pm 0.41$}
 \\

 \hline
  \small{Quadratic Spline (C)} & 0.64\tiny{$\pm 0.01$}  & 12.80\tiny{$\pm 0.02$} & -15.35\tiny{$\pm 0.02$} & -9.35\tiny{$\pm 0.44$} & 157.65\tiny{$\pm 0.28$}
  \\ 
  \hline
  \small{Quadratic Spline (AR)} & \textbf{0.66}\tiny{$\pm 0.01$}  & 12.91\tiny{$\pm 0.02$}  & -14.67\tiny{$\pm 0.03$} & -9.72\tiny{$\pm 0.47$} & 157.42\tiny{$\pm 0.28$}
  \\ 
 \hline
  \small{Cubic Spline} & 0.65\tiny{$\pm 0.01$}  & \textbf{13.14}\tiny{$\pm 0.02$}  & -14.59\tiny{$\pm 0.02$} & -9.06\tiny{$\pm 0.48$} & 157.24\tiny{$\pm 0.07$}
  \\ 
  \hline
  \small{RQ-NSF(C)} & 0.64\tiny{$\pm 0.01$}  & 13.09\tiny{$\pm 0.02$}  & -14.75\tiny{$\pm 0.03$} & -9.67\tiny{$\pm 0.47$} & 157.54\tiny{$\pm 0.28$}
  \\ 
   \hline
  \small{RQ-NSF(AR)} & \textbf{0.66}\tiny{$\pm 0.01$}  & 13.09\tiny{$\pm 0.02$}  & -14.01\tiny{$\pm 0.03$} & -9.22\tiny{$\pm 0.48$} & 157.31\tiny{$\pm 0.28$}
  \\ 
 [1ex] 
 \hline
\end{tabular}
\end{table*}

\subsection{Image datasets}

These datasets summarized in Table \ref{table:image}.  They are of increasing complexity and are preprocessed as in \cite{Dinh2017} by dequantizing with uniform noise (except for Flow++).

\begin{table}[!t]
\caption{Image datasets: data dimensionality and number of training examples for MNIST, CIFAR-10, ImageNet32 and ImageNet64 datasets.}
\label{table:image}
\centering
 \begin{tabular}{||c c c c c ||} 
 \hline
  & MNIST & CIFAR-10 & ImNet32 & ImNet64 \\ [0.5ex]
 \hline\hline
 Dims & 784 & 3072 & 3072 & 12288  \\ 
 \hline
 \#Train & 50K & $90$K & $\approx 1.3$M & $\approx 1.3$M \\ [1ex] 
 \hline
\end{tabular}
\end{table}

Table \ref{table:CompImage} compares performance on the image datasets for unconditional density estimation. For experimental details, see:  RealNVP for CIFAR-10 and ImageNet \citep{Dinh2017}, Glow for CIFAR-10 and ImageNet \citep{Kingma2018}, RealNVP and Glow for MNIST, MAF and FFJORD \citep{GrathwohlFFJORD:MODELS}, SOS \citep{Priyank2019}, RQ-NSF \citep{Durkan2019Neural}, UMNN \citep{Wehenkel2019}, iResNet \citep{Behrmann}, Residual Flow \citep{iresnetunbiased}, Flow++ \citep{ho2019flow}.

\begin{table}[!t]
\caption{Average test negative log-likelihood (in bits per dimension) for density estimation on image  datasets (lower is better).  }
\label{table:CompImage}
\centering
 \begin{tabular}{||c c c c c ||} 
 \hline
  & MNIST & CIFAR-10 & ImNet32 & ImNet64 \\ [0.5ex] 
 \hline\hline
 RealNVP & 1.06  & 3.49 & 4.28  & 3.98
 \\ 
 \hline
  Glow & 1.05  & 3.35 & 4.09 & 3.81
 \\ 
 \hline
 MAF & 1.89 & 4.31  & &
 \\ 
 \hline
  FFJORD & 0.99 & 3.40  & &
 \\ 
 \hline
 
  SOS & 1.81 & 4.18 & &
 \\ 
 \hline
  RQ-NSF(C) &    & 3.38 &  & 3.82 
  \\ 
  \hline
  UMNN & 1.13   &       &   &  
  \\ 
 \hline
  iResNet & 1.06 & 3.45  & &
 \\ 
 \hline
  Residual Flow \hspace{-1.2cm}& 0.97   & 3.28 & 4.01 & 3.76
  \\ 
 \hline
  Flow++ &   & \textbf{3.08}  & \textbf{3.86} & \textbf{3.69}
 \\ 
  [1ex] 
 \hline
\end{tabular}
\end{table}

As of this writing Flow++ \citep{ho2019flow} is the best performing approach. Besides using more expressive coupling layers (see Section \ref{s:CDF}) and a different architecture for the conditioner, variational dequantization was used instead of uniform.
An ablation study shows that the change in dequantization approach gave the most significant improvement.

\section{Discussion and open problems}
\label{discussion}

\subsection{Inductive biases}
\subsubsection{Role of the base measure}
The base measure of a normalizing flow is generally assumed to be a simple distribution (e.g., uniform or Gaussian).
However this doesn't need to be the case.
Any distribution where we can easily draw samples and compute the log probability density function is possible and the parameters of this distribution can be learned during training.

Theoretically the base measure shouldn't matter: any distribution for which a CDF can be computed, can be simulated by applying the inverse CDF to draw from the uniform distribution.
However in practice if structure is provided in the base measure, the resulting transformations may become easier to learn.
In other words, the choice of base measure can be viewed as a form of prior or inductive bias on the distribution and may be useful in its own right. For example, 
a trade-off between the complexity of the generative transformation and the form of base measure was explored in \citep{tails} in the context of modelling tail behaviour. 

\subsubsection{Form of diffeomorphisms} 
The majority of the flows explored are triangular flows (either coupling or autoregressive architectures).
Residual networks and Neural ODEs are also being actively investigated and applied.
A natural question to ask is: are there other ways to model diffeomorphisms which are efficient for computation? What inductive bias does the architecture impose? For instance, \cite{Spantini} investigate the relation between the sparsity of the triangular flow and Markov property of the target distribution.

A related question concerns the best way to model conditional normalizing flows when one needs to learn a conditional probability distribution. \cite{Trippe2018ConditionalDE} suggested using different flows for each condition, but this approach doesn't leverage weight sharing, and so is inefficient in terms of memory and data usage.  \cite{Atanov2019} proposed using affine coupling layers where the parameters $\theta$ depend on the condition. Conditional distributions are useful in particular for time series modelling, where one needs to find $p(\obsvali_t|\obsval_{<t})$ \citep{videoflow}.

\subsubsection{Loss function} The majority of the existing flows are trained by minimization of KL-divergence between source and the target distributions (or, equivalently, with log-likelihood maximization). However, other losses could be used which would put normalizing flows in a broader context of optimal transport theory \citep{transportation}. Interesting work has been done in this direction including Flow-GAN \citep{Grover2018} and the minimization of the Wasserstein distance as suggested by \citep{Arjovsky2017WassersteinG, Tolstikhin2018WassersteinA}.

\subsection{Generalisation to non-Euclidean spaces}\label{s:nonEuclid}
\subsubsection{Flows on manifolds.}  
Modelling probability distributions on manifolds has applications in many fields including robotics, molecular biology, optics, fluid mechanics, and plasma physics \citep{Gemici2016, Rezende2020tori}.  
How best to construct a normalizing flow on a general differentiable manifold remains an open question.
One approach to applying the normalizing flow framework on manifolds, is to find a base distribution on the Euclidean space and transfer it to the manifold of interest.
There are two main approaches: 1) embed the manifold in the Euclidean space and ``restrict" the measure, or 2) induce the measure from the tangent space to the manifold. We will briefly discuss each in turn.

One can also use differential structure to define measures on manifolds \citep{Spivak}.
Every differentiable and orientable  manifold $M$ has a volume form $\omega$, then for a Borel subset $U \subset M$ one can define its measure as $\mu_\omega(U) = \int_U \omega$. A Riemannian manifold has a natural volume form given by its metric tensor: $\omega = \sqrt{|g|}dx_1\wedge \dots \wedge dx_\flowdim$. \cite{Gemici2016} explore this approach considering an immersion of an $\flowdim$-dimensional manifold $M$ into a Euclidean space: $\phi: M \to \R^N$, where $N \ge \flowdim$.
In this case, one pulls-back a Euclidean metric, and locally a volume form on $M$ is $\omega = \sqrt{\det((\textrm{D}\phi)^T\textrm{D}\phi)} dx_1\wedge \dots \wedge dx_\flowdim$, where $\textrm{D}\phi$ is the Jacobian matrix of $\phi$.
\cite{Rezende2020tori} pointed out that the realization of this method is computationally hard, and proposed an alternative construction of flows on tori and spheres using diffeomorphisms of the one-dimensional circle as building blocks.  

As another option, one can consider exponential maps $\exp_x: T_xM \to M$, mapping a tangent space of a Riemannian manifold (at some point $x$) to the manifold itself.
If the manifold is geodesic complete, this map is globally defined, and locally is a diffeomorphism.
A tangent space has a structure of a vector space, so one can choose an isomorphism $T_xM \cong \R^\flowdim$. Then for a base distribution with the density $\basedensity$ on $\R^\flowdim$, one can push it forward on $M$ via the exponential map.
Additionally, applying a normalizing flow to a base measure before pushing it to $M$ helps to construct multimodal distributions on $M$. 
If the manifold $M$ is a hyberbolic space, the exponential map is a global diffeomorphism and all the formulas could be written explicitly. Using this method, \cite{hyperbolic2019} introduced the Gaussian reparameterization trick in a hyperbolic space and \cite{Bose2020hyperbolicflow} constructed hyperbolic normalizing flows.

Instead of a Riemannian structure, one can impose a Lie group structure on a manifold $G$.
In this case there also exists an exponential map $\exp: \mathfrak{g} \to G$ mapping a Lie algebra to the Lie group and one can use it to construct a normalizing flow on $G$.
\cite{lie2019} introduced an analog of the Gaussian reparameterization trick for a Lie group.

\subsubsection{Discrete distributions}
Modelling distributions over discrete spaces is important in a range of problems, however the generalization of normalizing flows to discrete distributions remains an open problem in practice.
Discrete latent variables were used by \citet{dinh2019rad} as an auxiliary tool to pushforward continuous random variables along piecewise-bijective maps (see Section \ref{piecewise-bijective}).
However, can we define normalizing flows if one or both of our distributions are discrete?
This could be useful for many applications including natural language modelling, graph generation and others. 

To this end \cite{Tran2019} model bijective functions on a finite set and show that, in this case, the change of variables is given by the formula: $\obsdensity(\obsval) = \basedensity(\flowfunc^{-1}(\obsval))$, \ie, with no Jacobian term (compare with Definition \ref{d:pushforward}).
For backpropagation of functions with discrete variables they use the straight-through gradient estimator \citep{Bengio2013StraightThrough}.
However this method is not scalable to distributions with large numbers of elements. 

Alternatively \cite{hoogeb2019lossless} models bijections on $\mathbb{Z}^\flowdim$ directly with additive coupling layers. Other approaches transform a discrete variable into a continuous latent variable with a variational autoencoder, and then apply normalizing flows in the continuous latent space \citep{Zizhuang2019, Ziegler2019}.

A different approach is dequantization, (\ie, adding noise to discrete data to make it continuous) which can be used with ordinal variables, \eg, discretized pixel intensities.
The noise can be uniform but other forms are possible and this dequantization can even be learned as a latent variable model \citep{ho2019flow, Hoogeboom2020dequant}. \cite{Hoogeboom2020dequant} analyzed how different choices of dequantization objectives and dequantization distributions affect the performance.

\ifCLASSOPTIONcompsoc
  \section*{Acknowledgments}
\else
  \section*{Acknowledgment}
\fi

The authors would like to thank Matt Taylor and Kry Yik-Chau Lui for their insightful comments.

\ifCLASSOPTIONcaptionsoff
  \newpage
\fi
\bibliographystyle{IEEEtranSN}

\begin{thebibliography}{108}
\providecommand{\natexlab}[1]{#1}
\providecommand{\url}[1]{#1}
\csname url@samestyle\endcsname
\providecommand{\newblock}{\relax}
\providecommand{\bibinfo}[2]{#2}
\providecommand{\BIBentrySTDinterwordspacing}{\spaceskip=0pt\relax}
\providecommand{\BIBentryALTinterwordstretchfactor}{4}
\providecommand{\BIBentryALTinterwordspacing}{\spaceskip=\fontdimen2\font plus
\BIBentryALTinterwordstretchfactor\fontdimen3\font minus
  \fontdimen4\font\relax}
\providecommand{\BIBforeignlanguage}[2]{{%
\expandafter\ifx\csname l@#1\endcsname\relax
\typeout{** WARNING: IEEEtranSN.bst: No hyphenation pattern has been}%
\typeout{** loaded for the language `#1'. Using the pattern for}%
\typeout{** the default language instead.}%
\else
\language=\csname l@#1\endcsname
\fi
#2}}
\providecommand{\BIBdecl}{\relax}
\BIBdecl

\bibitem[Abdelhamed et~al.(2019)Abdelhamed, Brubaker, and
  Brown]{Abdelhamed2019}
A.~Abdelhamed, M.~A. Brubaker, and M.~S. Brown, ``Noise flow: Noise modeling
  with conditional normalizing flows,'' in \emph{Proceedings of the IEEE
  International Conference on Computer Vision}, 2019, pp. 3165--3173.

\bibitem[Agnelli et~al.(2010)Agnelli, Cadeiras, Tabak, Cristina, and
  Vanden-Eijnden]{Agnelli2010}
J.~Agnelli, M.~Cadeiras, E.~Tabak, T.~Cristina, and E.~Vanden-Eijnden,
  ``Clustering and classification through normalizing flows in feature space,''
  \emph{Multiscale Modeling and Simulation}, vol.~8, pp. 1784--1802, 2010.

\bibitem[Arango and G{\'o}mez(2002)]{time_one_maps}
J.~Arango and A.~G{\'o}mez, ``{Diffeomorphisms as time one maps},''
  \emph{Aequationes Math.}, vol.~64, pp. 304--314, 2002.

\bibitem[Arjovsky et~al.(2017)Arjovsky, Chintala, and
  Bottou]{Arjovsky2017WassersteinG}
M.~Arjovsky, S.~Chintala, and L.~Bottou, ``{Wasserstein Generative Adversarial
  Networks},'' in \emph{ICML}, 2017.

\bibitem[Arnold(1978)]{ArnoldODE}
V.~Arnold, \emph{{Ordinary Differential Equations}}.\hskip 1em plus 0.5em minus
  0.4em\relax The MIT Press, 1978.

\bibitem[Atanov et~al.(2019)Atanov, Volokhova, Ashukha, Sosnovik, and
  Vetrov]{Atanov2019}
A.~Atanov, A.~Volokhova, A.~Ashukha, I.~Sosnovik, and D.~Vetrov,
  ``{Semi-Conditional Normalizing Flows for Semi-Supervised Learning},'' in
  \emph{Workshop on Invertible Neural Nets and Normalizing Flows, ICML}, 2019.

\bibitem[Behrmann et~al.(2019)Behrmann, Duvenaud, and Jacobsen]{Behrmann}
J.~Behrmann, D.~Duvenaud, and J.-H. Jacobsen, ``Invertible residual networks,''
  in \emph{Proceedings of the 36th International Conference on Machine
  Learning, ICML}, 2019.

\bibitem[Bengio et~al.(2013)Bengio, L{\'e}onard, and
  Courville]{Bengio2013StraightThrough}
Y.~Bengio, N.~L{\'e}onard, and A.~Courville, ``Estimating or propagating
  gradients through stochastic neurons for conditional computation,''
  \emph{arXiv preprint, arXiv:1308.3432}, 2013.

\bibitem[Bogachev et~al.(2005)Bogachev, Kolesnikov, and
  Medvedev]{triangular2005}
V.~Bogachev, A.~Kolesnikov, and K.~Medvedev, ``{Triangular transformations of
  measures},'' \emph{Sbornik Math.}, vol. 196, no. 3-4, pp. 309--335, 2005.

\bibitem[Bose et~al.(2020)Bose, Smofsky, Liao, Panangaden, and
  Hamilton]{Bose2020hyperbolicflow}
A.~J. Bose, A.~Smofsky, R.~Liao, P.~Panangaden, and W.~L. Hamilton, ``{Latent
  Variable Modelling with Hyperbolic Normalizing Flows},'' \emph{arXiv
  preprint, arXiv:2002.06336}, 2020.

\bibitem[Bowman et~al.(2015)Bowman, Vilnis, Vinyals, Dai, J{\'o}zefowicz, and
  Bengio]{Bowman2015GeneratingSF}
S.~R. Bowman, L.~Vilnis, O.~Vinyals, A.~M. Dai, R.~J{\'o}zefowicz, and
  S.~Bengio, ``Generating sentences from a continuous space,'' in \emph{CoNLL},
  2015.

\bibitem[Chang et~al.(2018)Chang, Meng, Haber, Ruthotto, Begert, and
  Holtham]{Chang2018}
B.~Chang, L.~Meng, E.~Haber, L.~Ruthotto, D.~Begert, and E.~Holtham,
  ``{Reversible Architectures for Arbitrarily Deep Residual Neural Networks},''
  in \emph{AAAI}, 2018.

\bibitem[Chang et~al.(2019)Chang, Chen, Haber, and
  Chi]{chang2018antisymmetricrnn}
B.~Chang, M.~Chen, E.~Haber, and E.~H. Chi, ``Antisymmetric{RNN}: A dynamical
  system view on recurrent neural networks,'' in \emph{ICLR}, 2019.

\bibitem[Chen et~al.(2018{\natexlab{b}})Chen, Li, Chen, Wang, Pu, and
  Carin]{Chen2018CTF}
C.~Chen, C.~Li, L.~Chen, W.~Wang, Y.~Pu, and L.~Carin, ``{Continuous-Time Flows
  for Efficient Inference and Density Estimation},'' in \emph{ICML}, 2018.

\bibitem[Chen et~al.(2018{\natexlab{a}})Chen, Rubanova, Bettencourt, and
  Duvenaud]{chen2018neural}
R.~T.~Q. Chen, Y.~Rubanova, J.~Bettencourt, and D.~Duvenaud, ``Neural ordinary
  differential equations,'' \emph{Advances in Neural Information Processing
  Systems}, 2018.

\bibitem[Chen et~al.(2019)Chen, Behrmann, Duvenaud, and
  Jacobsen]{iresnetunbiased}
R.~T.~Q. Chen, J.~Behrmann, D.~Duvenaud, and J.-H. Jacobsen, ``{Residual Flows
  for Invertible Generative Modeling},'' \emph{Advances in Neural Information
  Processing Systems}, 2019.

\bibitem[Creswell et~al.(2018)Creswell, White, Dumoulin, Arulkumaran, Sengupta,
  and Bharath]{Creswell2018GenerativeAN}
A.~Creswell, T.~White, V.~Dumoulin, K.~Arulkumaran, B.~Sengupta, and A.~A.
  Bharath, ``Generative adversarial networks: An overview,'' \emph{IEEE Signal
  Processing Magazine}, vol.~35, pp. 53--65, 2018.

\bibitem[Das et~al.(2019)Das, Abbeel, and Spanos]{Das2019}
H.~P. Das, P.~Abbeel, and C.~J. Spanos, ``{Dimensionality Reduction Flows},''
  \emph{arXiv preprint, arXiv:1908.01686}, 2019.

\bibitem[Dinh et~al.(2015)Dinh, Krueger, and Bengio]{Dinh2015}
L.~Dinh, D.~Krueger, and Y.~Bengio, ``{NICE: Non-linear Independent Components
  Estimation},'' in \emph{ICLR Workshop}, 2015.

\bibitem[Dinh et~al.(2017)Dinh, Sohl-Dickstein, and Bengio]{Dinh2017}
L.~Dinh, J.~Sohl-Dickstein, and S.~Bengio, ``{Density Estimation using Real
  NVP},'' in \emph{ICLR}, 2017.

\bibitem[Dinh et~al.(2019)Dinh, Sohl-Dickstein, Pascanu, and
  Larochelle]{dinh2019rad}
L.~Dinh, J.~Sohl-Dickstein, R.~Pascanu, and H.~Larochelle, ``{A RAD approach to
  deep mixture models},'' in \emph{ICLR Workshop}, 2019.

\bibitem[Dua and Graff(2017)]{uci}
D.~Dua and C.~Graff, ``{UCI Machine Learning Repository},'' 2017.

\bibitem[Dupont et~al.(2019)Dupont, Doucet, and Teh]{augmentedNeuralODE}
E.~Dupont, A.~Doucet, and Y.~W. Teh, ``{Augmented Neural ODEs},''
  \emph{Advances in Neural Information Processing Systems}, 2019.

\bibitem[Durkan et~al.(2019{\natexlab{a}})Durkan, Bekasov, Murray, and
  Papamakarios]{Durkan2019Cubic}
C.~Durkan, A.~Bekasov, I.~Murray, and G.~Papamakarios, ``{Cubic-spline
  flows},'' in \emph{Workshop on Invertible Neural Networks and Normalizing
  Flows, ICML}, 2019.

\bibitem[Durkan et~al.(2019{\natexlab{b}})Durkan, Bekasov, Murray, and
  Papamakarios]{Durkan2019Neural}
------, ``{Neural Spline Flows},'' \emph{Advances in Neural Information
  Processing Systems}, 2019.

\bibitem[E(2017)]{e2017}
W.~E, ``A proposal on machine learning via dynamical systems,''
  \emph{Communications in Mathematics and Statistics}, vol.~5, pp. 1--11, 2017.

\bibitem[Esling et~al.(2019)Esling, Masuda, Bardet, Despres, and
  Chemla-Romeu-Santos]{esling2019}
P.~Esling, N.~Masuda, A.~Bardet, R.~Despres, and A.~Chemla-Romeu-Santos,
  ``Universal audio synthesizer control with normalizing flows,'' \emph{arXiv
  preprint, arXiv:1907.00971}, 2019.

\bibitem[Falorsi et~al.(2019)Falorsi, de~Haan, Davidson, and
  Forr{\'e}]{lie2019}
L.~Falorsi, P.~de~Haan, T.~R. Davidson, and P.~Forr{\'e}, ``{Reparameterizing
  Distributions on Lie Groups},'' \emph{arXiv preprint, arXiv:1903.02958},
  2019.

\bibitem[Finlay et~al.(2020)Finlay, Jacobsen, Nurbekyan, and
  Oberman]{Finlay2020trainNODE}
C.~Finlay, J.-H. Jacobsen, L.~Nurbekyan, and A.~M. Oberman, ``{How to train
  your neural ODE},'' \emph{arXiv preprint, arXiv:2002.02798}, 2020.

\bibitem[Gemici et~al.(2016)Gemici, Rezende, and Mohamed]{Gemici2016}
M.~C. Gemici, D.~Rezende, and S.~Mohamed, ``{Normalizing Flows on Riemannian
  Manifolds},'' \emph{arXiv preprint, arXiv:1611.02304}, 2016.

\bibitem[Germain et~al.(2015)Germain, Gregor, Murray, and
  Larochelle]{Germain2015}
M.~Germain, K.~Gregor, I.~Murray, and H.~Larochelle, ``{MADE: Masked
  Autoencoder for Distribution Estimation},'' in \emph{ICML}, 2015.

\bibitem[Gomez et~al.(2017)Gomez, Ren, Urtasun, and Grosse]{Gomez2017}
A.~N. Gomez, M.~Ren, R.~Urtasun, and R.~B. Grosse, ``{The Reversible Residual
  Network: Backpropagation Without Storing Activations},'' \emph{Advances in
  Neural Information Processing Systems}, 2017.

\bibitem[Goodfellow et~al.(2014)Goodfellow, Pouget-Abadie, Mirza, Xu,
  Warde-Farley, Ozair, Courville, and Bengio]{Goodfellow2014GenerativeAN}
I.~J. Goodfellow, J.~Pouget-Abadie, M.~Mirza, B.~Xu, D.~Warde-Farley, S.~Ozair,
  A.~C. Courville, and Y.~Bengio, ``{Generative Adversarial Nets},''
  \emph{Advances in Neural Information Processing Systems}, 2014.

\bibitem[Grathwohl et~al.(2019)Grathwohl, Q~Chen, Bettencourt, Sutskever, and
  Duvenaud]{GrathwohlFFJORD:MODELS}
W.~Grathwohl, R.~T. Q~Chen, J.~Bettencourt, I.~Sutskever, and D.~Duvenaud,
  ``{FFJORD: Free-form continuous dynamics for scalable reversible generative
  models},'' in \emph{ICLR}, 2019.

\bibitem[Gregory and Delbourgo(1982)]{Gregory1982}
J.~Gregory and R.~Delbourgo, ``Piecewise rational quadratic interpolation to
  monotonic data,'' \emph{IMA Journal of Numerical Analysis}, vol.~2, no.~2,
  pp. 123--130, 1982.

\bibitem[Grover et~al.(2018)Grover, Dhar, and Ermon]{Grover2018}
A.~Grover, M.~Dhar, and S.~Ermon, ``{Flow-GAN: Combining Maximum Likelihood and
  Adversarial Learning in Generative Models},'' in \emph{AAAI}, 2018.

\bibitem[Haber et~al.(2018)Haber, Ruthotto, and Holtham]{haber2017}
E.~Haber, L.~Ruthotto, and E.~Holtham, ``Learning across scales - a multiscale
  method for convolution neural networks,'' in \emph{AAAI}, 2018.

\bibitem[Hasenclever et~al.(2017)Hasenclever, Tomczak, Van Den~Berg, and
  Welling]{Hasenclever2017}
L.~Hasenclever, J.~M. Tomczak, R.~Van Den~Berg, and M.~Welling, ``{Variational
  Inference with Orthogonal Normalizing Flows},'' in \emph{Workshop on Bayesian
  Deep Learning, NIPS}, 2017.

\bibitem[He et~al.(2015)He, Zhang, Ren, and Sun]{He2015}
K.~He, X.~Zhang, S.~Ren, and J.~Sun, ``{Delving Deep into Rectifiers:
  Surpassing Human-Level Performance on ImageNet Classification},'' in
  \emph{ICCV}, 2015.

\bibitem[He et~al.(2016)He, Zhang, Ren, and Sun]{He2016}
------, ``{Deep Residual Learning for Image Recognition},'' in \emph{CVPR},
  2016.

\bibitem[Ho et~al.(2019)Ho, Chen, Srinivas, Duan, and Abbeel]{ho2019flow}
J.~Ho, X.~Chen, A.~Srinivas, Y.~Duan, and P.~Abbeel, ``Flow++: Improving
  flow-based generative models with variational dequantization and architecture
  design,'' in \emph{Proceedings of the 36th International Conference on
  Machine Learning, ICML}, 2019.

\bibitem[Hoogeboom et~al.(2019{\natexlab{a}})Hoogeboom, Berg, and
  Welling]{Hoogeboom2019}
E.~Hoogeboom, R.~V.~D. Berg, and M.~Welling, ``{Emerging Convolutions for
  Generative Normalizing Flows},'' in \emph{Proceedings of the 36th
  International Conference on Machine Learning, ICML}, 2019.

\bibitem[Hoogeboom et~al.(2019{\natexlab{b}})Hoogeboom, Peters, van~den Berg,
  and Welling]{hoogeb2019lossless}
E.~Hoogeboom, J.~W. Peters, R.~van~den Berg, and M.~Welling, ``Integer discrete
  flows and lossless compression,'' in \emph{NeurIPS}, 2019.

\bibitem[Hoogeboom et~al.(2020)Hoogeboom, Cohen, and
  Tomczak]{Hoogeboom2020dequant}
E.~Hoogeboom, T.~S. Cohen, and J.~M. Tomczak, ``Learning discrete distributions
  by dequantization,'' \emph{arXiv preprint, arXiv:2001.11235}, 2020.

\bibitem[Huang et~al.(2018)Huang, Krueger, Lacoste, and
  Courville]{Huang2018NeuralFlows}
C.-W. Huang, D.~Krueger, A.~Lacoste, and A.~Courville, ``{Neural Autoregressive
  Flows},'' in \emph{ICML}, 2018.

\bibitem[Ioffe and Szegedy(2015)]{Ioffe2015}
S.~Ioffe and C.~Szegedy, ``{Batch Normalization: Accelerating Deep Network
  Training by Reducing Internal Covariate Shift},'' in \emph{ICML}, 2015.

\bibitem[Jacobsen et~al.(2018)Jacobsen, Smeulders, and
  Oyallon]{jacobsen2018irevnet}
J.-H. Jacobsen, A.~W. Smeulders, and E.~Oyallon, ``{i-RevNet: Deep Invertible
  Networks},'' in \emph{ICLR}, 2018.

\bibitem[Jaini et~al.(2019{\natexlab{b}})Jaini, Kobyzev, Brubaker, and
  Yu]{tails}
P.~Jaini, I.~Kobyzev, M.~Brubaker, and Y.~Yu, ``{Tails of Triangular Flows},''
  \emph{arXiv preprint, arXiv:1907.04481}, 2019.

\bibitem[Jaini et~al.(2019{\natexlab{a}})Jaini, Selby, and Yu]{Priyank2019}
P.~Jaini, K.~A. Selby, and Y.~Yu, ``{Sum-of-squares polynomial flow},'' in
  \emph{Proceedings of the 36th International Conference on Machine Learning,
  ICML}, 5 2019.

\bibitem[Jankowiak and Obermeyer(2018)]{Jankowiak2018}
M.~Jankowiak and F.~Obermeyer, ``Pathwise derivatives beyond the
  reparameterization trick,'' in \emph{Proceedings of the 35th International
  Conference on Machine Learning, ICML}, 2018.

\bibitem[Kanwar et~al.(2020)Kanwar, Albergo, Boyda, Cranmer, Hackett,
  Racani{\`e}re, Rezende, and Shanahan]{Kanwar2020}
G.~Kanwar, M.~S. Albergo, D.~Boyda, K.~Cranmer, D.~C. Hackett,
  S.~Racani{\`e}re, D.~J. Rezende, and P.~E. Shanahan, ``Equivariant flow-based
  sampling for lattice gauge theory,'' \emph{arXiv preprint, arXiv:2003.06413},
  2020.

\bibitem[Katok and Hasselblatt(1995)]{intro_to_dynamical_systems}
A.~Katok and B.~Hasselblatt, \emph{{ Introduction to the modern theory of
  dynamical systems}}.\hskip 1em plus 0.5em minus 0.4em\relax Cambridge
  University Press, New York, 1995.

\bibitem[Kim et~al.(2018)Kim, gil Lee, Song, Kim, and Yoon]{flowwavenet}
S.~Kim, S.~gil Lee, J.~Song, J.~Kim, and S.~Yoon, ``{FloWaveNet: A Generative
  Flow for Raw Audio},'' in \emph{Proceedings of the 36th International
  Conference on Machine Learning, ICML}, 2018.

\bibitem[Kingma and Welling(2014)]{Kingma2014AutoEncodingVB}
D.~P. Kingma and M.~Welling, ``Auto-encoding variational bayes,'' in
  \emph{Proceedings of the 2nd International Conference on Learning
  Representations, ICLR}, 2014.

\bibitem[Kingma and Welling(2019)]{Kingma2019IntroVAE}
------, ``{An Introduction to Variational Autoencoders},'' \emph{arXiv
  preprint, arXiv:1906.02691}, 2019.

\bibitem[Kingma et~al.(2016)Kingma, Salimans, Jozefowicz, Chen, Sutskever, and
  Welling]{Kingma2016}
D.~P. Kingma, T.~Salimans, R.~Jozefowicz, X.~Chen, I.~Sutskever, and
  M.~Welling, ``{Improved Variational Inference with Inverse Autoregressive
  Flow},'' in \emph{NIPS}, 2016.

\bibitem[Kingma and Dhariwal(2018)]{Kingma2018}
D.~P. Kingma and P.~Dhariwal, ``Glow: Generative flow with invertible 1x1
  convolutions,'' in \emph{Advances in Neural Information Processing Systems},
  2018, pp. 10\,215--10\,224.

\bibitem[K{\"o}hler et~al.(2019)K{\"o}hler, Klein, and No{\'e}]{Koller2019Eq}
J.~K{\"o}hler, L.~Klein, and F.~No{\'e}, ``Equivariant flows: sampling
  configurations for multi-body systems with symmetric energies,'' in
  \emph{Workshop on Machine Learning and the Physical Sciences, NeurIPS}, 2019.

\bibitem[Koller and Friedman(2009)]{graphModels}
D.~Koller and N.~Friedman, \emph{{Probabilistic Graphical Models}}.\hskip 1em
  plus 0.5em minus 0.4em\relax Massachusetts: MIT Press, 2009.

\bibitem[Kumar et~al.(2019)Kumar, Babaeizadeh, Erhan, Finn, Levine, Dinh, and
  Kingma]{videoflow}
M.~Kumar, M.~Babaeizadeh, D.~Erhan, C.~Finn, S.~Levine, L.~Dinh, and D.~Kingma,
  ``{VideoFlow: A Flow-Based Generative Model for Video},'' in \emph{Workshop
  on Invertible Neural Nets and Normalizing Flows, ICML}, 2019.

\bibitem[Laurence et~al.(2014)Laurence, Pignol, and Tabak]{Laurence2014}
P.~M. Laurence, R.~J. Pignol, and E.~G. Tabak, ``Constrained density
  estimation,'' \emph{Proceedings of the 2011 Wolfgang Pauli Institute
  conference on energy and commodity trading, Springer Verlag}, pp. 259--284,
  2014.

\bibitem[Li et~al.(2020)Li, Wong, Chen, and Duvenaud]{Li2020SDE}
X.~Li, T.-K.~L. Wong, R.~T.~Q. Chen, and D.~Duvenaud, ``{Scalable Gradients for
  Stochastic Differential Equations},'' \emph{arXiv preprint,
  arXiv:2001.01328}, 2020.

\bibitem[Liutkus et~al.(2019)Liutkus, Simsekli, Majewski, Durmus, and
  St{\"o}ter]{Liutkus2018SlicedWassersteinFN}
A.~Liutkus, U.~Simsekli, S.~Majewski, A.~Durmus, and F.-R. St{\"o}ter,
  ``{Sliced-Wasserstein Flows: Nonparametric Generative Modeling via Optimal
  Transport and Diffusions},'' in \emph{Proceedings of the 36th International
  Conference on Machine Learning, ICML}, 2019.

\bibitem[Maas et~al.(2013)Maas, Hannun, and Ng]{Maas2013}
A.~L. Maas, A.~Y. Hannun, and A.~Y. Ng, ``{Rectifier Nonlinearities Improve
  Neural Network Acoustic Models},'' in \emph{ICML}, 2013.

\bibitem[Madhawa et~al.(2019)Madhawa, Ishiguro, Nakago, and Abe]{Madhawa2019}
K.~Madhawa, K.~Ishiguro, K.~Nakago, and M.~Abe, ``{GraphNVP: An Invertible Flow
  Model for Generating Molecular Graphs},'' \emph{arXiv preprint,
  arXiv:1905.11600}, 2019.

\bibitem[Martin et~al.(2001)Martin, Fowlkes, Tal, and Malik]{bsds}
D.~Martin, C.~Fowlkes, D.~Tal, and J.~Malik, ``A database of human segmented
  natural images and its application to evaluating segmentation algorithms and
  measuring ecological statistics,'' in \emph{Proceedings of the 8th
  International Conference on Computer Vision, ICCV}, 2001.

\bibitem[Mazoure et~al.(2019)Mazoure, Doan, Durand, Pineau, and
  Hjelm]{Mazoure2019}
B.~Mazoure, T.~Doan, A.~Durand, J.~Pineau, and R.~D. Hjelm, ``Leveraging
  exploration in off-policy algorithms via normalizing flows,'' in \emph{3rd
  Conference on Robot Learning (CoRL 2019)}, 2019.

\bibitem[Medvedev(2008)]{triangular2008}
K.~V. Medvedev, ``{Certain properties of triangular transformations of
  measures},'' \emph{Theory Stoch. Process.}, vol. 14(30), pp. 95--99, 2008.

\bibitem[M{\"{u}}ller et~al.(2018)M{\"{u}}ller, McWilliams, Rousselle, Gross,
  and Novak]{Muller2018NeuralSampling}
T.~M{\"{u}}ller, B.~McWilliams, F.~Rousselle, M.~Gross, and J.~Novak, ``{Neural
  Importance Sampling},'' \emph{ACM Transactions on Graphics (TOG)}, vol.~38,
  2018.

\bibitem[Nadeem~Ward et~al.(2019)Nadeem~Ward, Smofsky, and
  Joey~Bose]{Ward2019RL}
P.~Nadeem~Ward, A.~Smofsky, and A.~Joey~Bose, ``Improving exploration in
  soft-actor-critic with normalizing flows policies,'' in \emph{Workshop on
  Invertible Neural Networks and Normalizing Flows, ICML}, 2019.

\bibitem[No{\'e} et~al.(2019)No{\'e}, Olsson, K{\"o}hler, and Wu]{Noe2019}
F.~No{\'e}, S.~Olsson, J.~K{\"o}hler, and H.~Wu, ``{Boltzmann generators:
  Sampling equilibrium states of many-body systems with deep learning},''
  \emph{Science}, vol. 365, 2019.

\bibitem[Oksendal(1992)]{Oksendal}
B.~Oksendal, \emph{Stochastic Differential Equations (3rd Ed.): An Introduction
  with Applications}.\hskip 1em plus 0.5em minus 0.4em\relax Berlin,
  Heidelberg: Springer-Verlag, 1992.

\bibitem[Ovinnikov(2018)]{hyperbolic2019}
I.~Ovinnikov, ``{Poincar{\'e} Wasserstein Autoencoder},'' in \emph{Bayesian
  Deep Learning Workshop, NeurIPS}, 2018.

\bibitem[Papamakarios et~al.(2017)Papamakarios, Pavlakou, and
  Murray]{Papamakarios2017}
G.~Papamakarios, T.~Pavlakou, and I.~Murray, ``{Masked Autoregressive Flow for
  Density Estimation},'' in \emph{NIPS}, 2017.

\bibitem[Papamakarios et~al.(2019)Papamakarios, Nalisnick, Rezende, Mohamed,
  and Lakshminarayanan]{FlowSurvey2019}
G.~Papamakarios, E.~Nalisnick, D.~J. Rezende, S.~Mohamed, and
  B.~Lakshminarayanan, ``{Normalizing Flows for Probabilistic Modeling and
  Inference},'' \emph{arXiv preprint, arXiv:1912.02762}, 2019.

\bibitem[Peluchetti and Favaro(2019)]{NSDE1}
S.~Peluchetti and S.~Favaro, ``{Neural Stochastic Differential Equations},''
  \emph{arXiv preprint, arXiv:1905.11065}, 2019.

\bibitem[Prenger et~al.(2019)Prenger, Valle, and Catanzaro]{waveglow}
R.~Prenger, R.~Valle, and B.~Catanzaro, ``Waveglow: A flow-based generative
  network for speech synthesis,'' in \emph{ICASSP}, 2019.

\bibitem[Rezende and Mohamed(2015)]{Rezende2015}
D.~J. Rezende and S.~Mohamed, ``{Variational Inference with Normalizing
  Flows},'' in \emph{ICML}, 2015.

\bibitem[Rezende et~al.(2020)Rezende, Papamakarios, Racani{\`e}re, Albergo,
  Kanwar, Shanahan, and Cranmer]{Rezende2020tori}
D.~J. Rezende, G.~Papamakarios, S.~Racani{\`e}re, M.~S. Albergo, G.~Kanwar,
  P.~E. Shanahan, and K.~Cranmer, ``{Normalizing Flows on Tori and Spheres},''
  \emph{arXiv preprint, arXiv:2002.02428}, 2020.

\bibitem[Rippel and Adams(2013)]{Rippel2013}
O.~Rippel and R.~P. Adams, ``High-dimensional probability estimation with deep
  density models,'' \emph{arXiv preprint arXiv:1302.5125}, 2013.

\bibitem[Salimans et~al.(2015)Salimans, Diederik, Kingma, and
  Welling]{Salimans2015}
T.~Salimans, A.~Diederik, D.~P. Kingma, and M.~Welling, ``{Markov Chain Monte
  Carlo and Variational Inference: Bridging the Gap},'' in \emph{ICML}, 2015.

\bibitem[Salimans et~al.(2016)Salimans, Goodfellow, Zaremba, Cheung, Radford,
  and Chen]{Salimans2016ImprovedTF}
T.~Salimans, I.~J. Goodfellow, W.~Zaremba, V.~Cheung, A.~Radford, and X.~Chen,
  ``{Improved Techniques for Training GANs},'' in \emph{NIPS}, 2016.

\bibitem[Salman et~al.(2018)Salman, Yadollahpour, Fletcher, and
  Batmanghelich]{Salman2018DeepDN}
H.~Salman, P.~Yadollahpour, T.~Fletcher, and N.~Batmanghelich, ``Deep
  diffeomorphic normalizing flows,'' \emph{arXiv preprint, arXiv:1810.03256},
  2018.

\bibitem[Sohl{-}Dickstein et~al.(2015)Sohl{-}Dickstein, Weiss, Maheswaranathan,
  and Ganguli]{Sohl-DicksteinW15}
J.~Sohl{-}Dickstein, E.~A. Weiss, N.~Maheswaranathan, and S.~Ganguli, ``Deep
  unsupervised learning using nonequilibrium thermodynamics,'' in
  \emph{Proceedings of the 32nd International Conference on Machine Learning,
  ICML}, 2015.

\bibitem[Spantini et~al.(2017)Spantini, Bigoni, and Marzouk]{Spantini}
A.~Spantini, D.~Bigoni, and Y.~Marzouk, ``Inference via low-dimensional
  couplings,'' \emph{Journal of Machine Learning Research}, vol.~19, 03 2017.

\bibitem[Spivak(1965)]{Spivak}
M.~Spivak, \emph{{Calculus on Manifolds: A Modern Approach to Classical
  Theorems of Advanced Calculus.}}\hskip 1em plus 0.5em minus 0.4em\relax
  Print, 1965.

\bibitem[Suykens et~al.(1998)Suykens, Verrelst, and
  Vandewalle]{Fokker-Plank-machine}
J.~Suykens, H.~Verrelst, and J.~Vandewalle, ``{On-Line Learning Fokker-Planck
  Machine},'' \emph{Neural Processing Letters}, vol.~7, pp. 81--89, 1998.

\bibitem[Tabak and Turner(2013)]{Tabak2013}
E.~G. Tabak and C.~V. Turner, ``{A Family of Nonparametric Density Estimation
  Algorithms},'' \emph{Communications on Pure and Applied Mathematics},
  vol.~66, no.~2, pp. 145--164, 2013.

\bibitem[Tabak and Vanden-Eijnden(2010)]{Tabak2010}
E.~G. Tabak and E.~Vanden-Eijnden, ``{Density Estimation by Dual Ascent of the
  Log-Likelihood},'' \emph{Communications in Mathematical Sciences}, vol.~8,
  no.~1, pp. 217--233, 2010.

\bibitem[Tolstikhin et~al.(2018)Tolstikhin, Bousquet, Gelly, and
  Sch{\"o}lkopf]{Tolstikhin2018WassersteinA}
I.~O. Tolstikhin, O.~Bousquet, S.~Gelly, and B.~Sch{\"o}lkopf, ``{Wasserstein
  Auto-Encoders},'' in \emph{ICLR}, 2018.

\bibitem[Tomczak and Welling(2017)]{Tomczakconvex}
J.~Tomczak and M.~Welling, ``{Improving Variational Auto-Encoders using convex
  combination linear Inverse Autoregressive Flow},'' \emph{Benelearn}, 2017.

\bibitem[Tomczak and Welling(2016)]{Tomczak2016}
J.~M. Tomczak and M.~Welling, ``Improving variational auto-encoders using
  householder flow,'' \emph{arXiv preprint arXiv:1611.09630}, 2016.

\bibitem[Touati et~al.(2019)Touati, Satija, Romoff, Pineau, and
  Vincent]{Touati2019}
A.~Touati, H.~Satija, J.~Romoff, J.~Pineau, and P.~Vincent, ``Randomized value
  functions via multiplicative normalizing flows,'' in \emph{UAI2019:
  Conference on Uncertainty in Artificial Intelligence}, 2019.

\bibitem[Tran et~al.(2019)Tran, Vafa, Agrawal, Dinh, and Poole]{Tran2019}
D.~Tran, K.~Vafa, K.~Agrawal, L.~Dinh, and B.~Poole, ``{Discrete Flows:
  Invertible Generative Models of Discrete Data},'' in \emph{ICLR Workshop},
  2019.

\bibitem[Trippe and Turner(2017)]{Trippe2018ConditionalDE}
B.~L. Trippe and R.~E. Turner, ``{Conditional Density Estimation with Bayesian
  Normalising Flows},'' in \emph{Workshop on Bayesian Deep Learning, NIPS},
  2017.

\bibitem[Tzen and Raginsky(2019)]{NSDE2}
B.~Tzen and M.~Raginsky, ``{Neural Stochastic Differential Equations: Deep
  Latent Gaussian Models in the Diffusion Limit},'' \emph{arXiv preprint,
  arXiv:1905.09883}, 2019.

\bibitem[van~den Berg et~al.(2018)van~den Berg, Hasenclever, Tomczak, and
  Welling]{sylvester}
R.~van~den Berg, L.~Hasenclever, J.~M. Tomczak, and M.~Welling, ``Sylvester
  normalizing flows for variational inference,'' in \emph{Proceedings of the
  34th Conference on Uncertainty in Artificial Intelligence, UAI}, 2018.

\bibitem[van~den Oord et~al.(2017)van~den Oord, Li, Babuschkin, Simonyan,
  Vinyals, Kavukcuoglu, van~den Driessche, Lockhart, Cobo, Stimberg,
  Casagrande, Grewe, Noury, Dieleman, Elsen, Kalchbrenner, Zen, Graves, King,
  Walters, Belov, and Hassabis]{Oord2017ParallelWF}
A.~van~den Oord, Y.~Li, I.~Babuschkin, K.~Simonyan, O.~Vinyals, K.~Kavukcuoglu,
  G.~van~den Driessche, E.~Lockhart, L.~C. Cobo, F.~Stimberg, N.~Casagrande,
  D.~Grewe, S.~Noury, S.~Dieleman, E.~Elsen, N.~Kalchbrenner, H.~Zen,
  A.~Graves, H.~King, T.~Walters, D.~Belov, and D.~Hassabis, ``Parallel
  wavenet: Fast high-fidelity speech synthesis,'' in \emph{ICML}, 2017.

\bibitem[Villani(2003)]{transportation}
C.~Villani, \emph{{Topics in optimal transportation (Graduate Studies in
  Mathematics 58)}}.\hskip 1em plus 0.5em minus 0.4em\relax American
  Mathematical Society, Providence, RI, 2003.

\bibitem[Wang et~al.(2017)Wang, Gou, Duan, Lin, Zheng, and yue
  Wang]{Wang2017GenerativeAN}
K.~Wang, C.~Gou, Y.~Duan, Y.~Lin, X.~Zheng, and F.~yue Wang, ``Generative
  adversarial networks: introduction and outlook,'' \emph{IEEE/CAA Journal of
  Automatica Sinica}, vol.~4, pp. 588--598, 2017.

\bibitem[Wang and Wang(2019)]{Zizhuang2019}
P.~Z. Wang and W.~Y. Wang, ``{Riemannian Normalizing Flow on Variational
  Wasserstein Autoencoder for Text Modeling},'' \emph{arXiv preprint,
  arXiv:1904.02399}, 2019.

\bibitem[Wehenkel and Louppe(2019)]{Wehenkel2019}
A.~Wehenkel and G.~Louppe, ``{Unconstrained Monotonic Neural Networks},''
  \emph{arXiv preprint, arXiv:1908.05164}, 2019.

\bibitem[Welling and Teh(2011)]{Welling2011BayesianLV}
M.~Welling and Y.~W. Teh, ``{Bayesian Learning via Stochastic Gradient Langevin
  Dynamics},'' in \emph{ICML}, 2011.

\bibitem[Wirnsberger et~al.(2020)Wirnsberger, Ballard, Papamakarios,
  Abercrombie, Racani{\`e}re, Pritzel, Jimenez~Rezende, and
  Blundell]{Wirnsberger2020}
P.~Wirnsberger, A.~Ballard, G.~Papamakarios, S.~Abercrombie, S.~Racani{\`e}re,
  A.~Pritzel, D.~Jimenez~Rezende, and C.~Blundell, ``Targeted free energy
  estimation via learned mappings,'' \emph{arXiv preprint, arXiv:2002.04913},
  2020.

\bibitem[Wong et~al.(2020)Wong, Contardo, and Ho]{KazeWong2020}
K.~W.~K. Wong, G.~Contardo, and S.~Ho, ``Gravitational wave population
  inference with deep flow-based generative network,'' \emph{arXiv preprint,
  arXiv:2002.09491}, 2020.

\bibitem[Zhang et~al.(2019)Zhang, Gao, Unterman, and Arodz]{Zhang2019ode}
H.~Zhang, X.~Gao, J.~Unterman, and T.~Arodz, ``{Approximation Capabilities of
  Neural Ordinary Differential Equations},'' \emph{arXiv preprint,
  arXiv:1907.12998}, 2019.

\bibitem[Zheng et~al.(2018)Zheng, Yang, and Carbonell]{Zheng2018}
G.~Zheng, Y.~Yang, and J.~Carbonell, ``{Convolutional Normalizing Flows},'' in
  \emph{Workshop on Theoretical Foundations and Applications of Deep Generative
  Models, ICML}, 2018.

\bibitem[Ziegler and Rush(2019)]{Ziegler2019}
Z.~M. Ziegler and A.~M. Rush, ``{Latent Normalizing Flows for Discrete
  Sequences},'' in \emph{Proceedings of the 36th International Conference on
  Machine Learning, ICML}, 2019.

\end{thebibliography}

\begin{IEEEbiography}[{\includegraphics[width=1in,height=1.25in,clip,keepaspectratio]{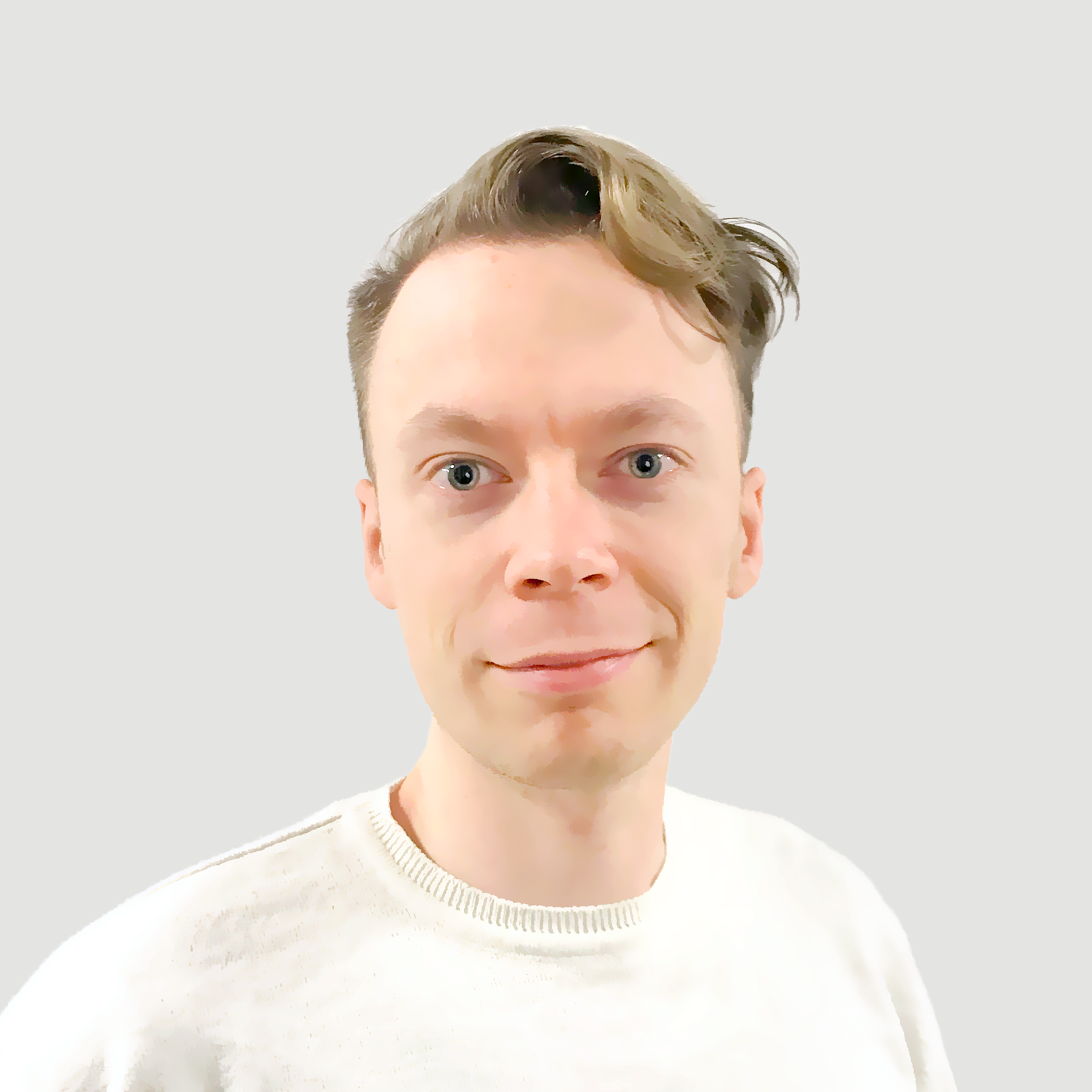}}]{Ivan Kobyzev}
Ivan Kobyzev received his Master’s degree in Mathematical Physics from St Petersburg State University, Russia, in 2011, and his PhD in Mathematics from Western University, Canada, in 2016. He did two postdocs in Mathematics and in Computer Science at the University of Waterloo. Currently he is a researcher at Borealis AI. His research interests include Algebra, Generative Models, Cognitive Computing, Natural Language Processing.   
\end{IEEEbiography}

\begin{IEEEbiography}[{\includegraphics[width=1in,height=1.25in,clip,keepaspectratio]{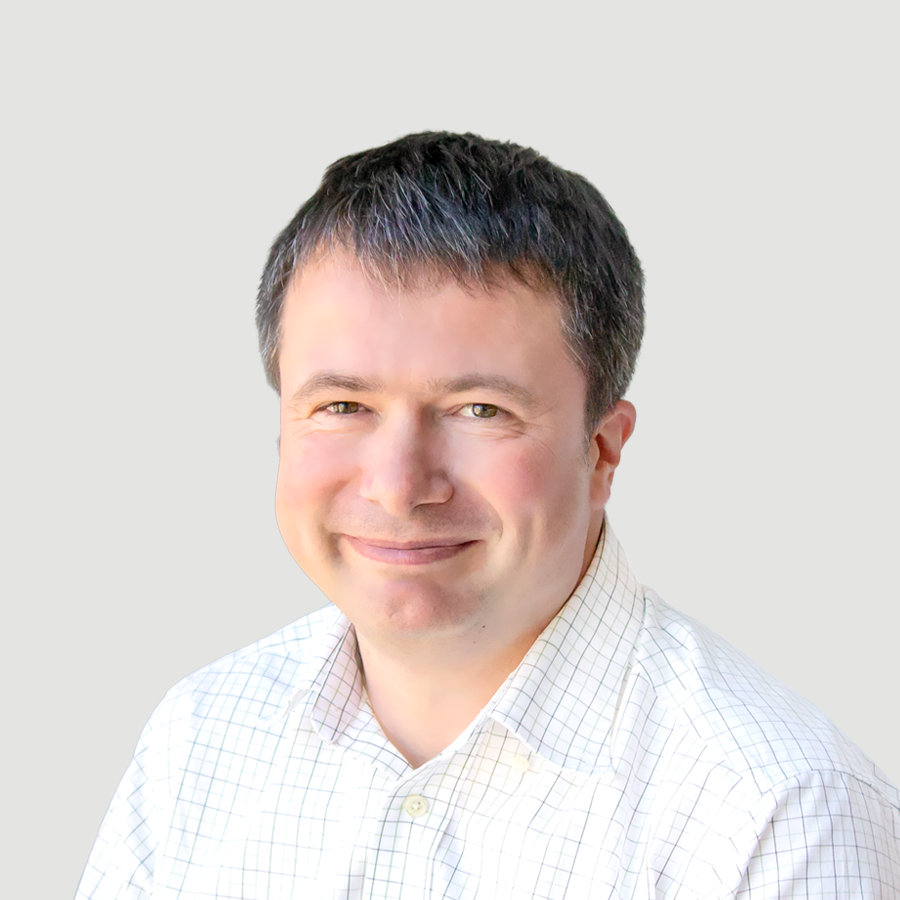}}]{Simon J.D. Prince}
Simon Prince holds a Masters by Research from University College London and a doctorate from the University of Oxford. He has a diverse research background and has published in wide-ranging areas including Computer Vision, Neuroscience, HCI, Computer Graphics, Medical Imaging, and Augmented Reality. He is also the author of a popular textbook on Computer Vision.  
From 2005-2012 Dr. Prince was a tenured faculty member in the Department of Computer Science at University College London, where he taught courses in Computer Vision, Image Processing and Advanced Statistical Methods. During this time, he was Director of the M.Sc. in Computer Vision, Graphics and Imaging. Dr. Prince worked in industry applying AI to computer graphics software.  Currently he is a Research Director of Borealis AI's Montreal office.
\end{IEEEbiography}

\begin{IEEEbiography}[{\includegraphics[width=1in,height=1.25in,clip,keepaspectratio]{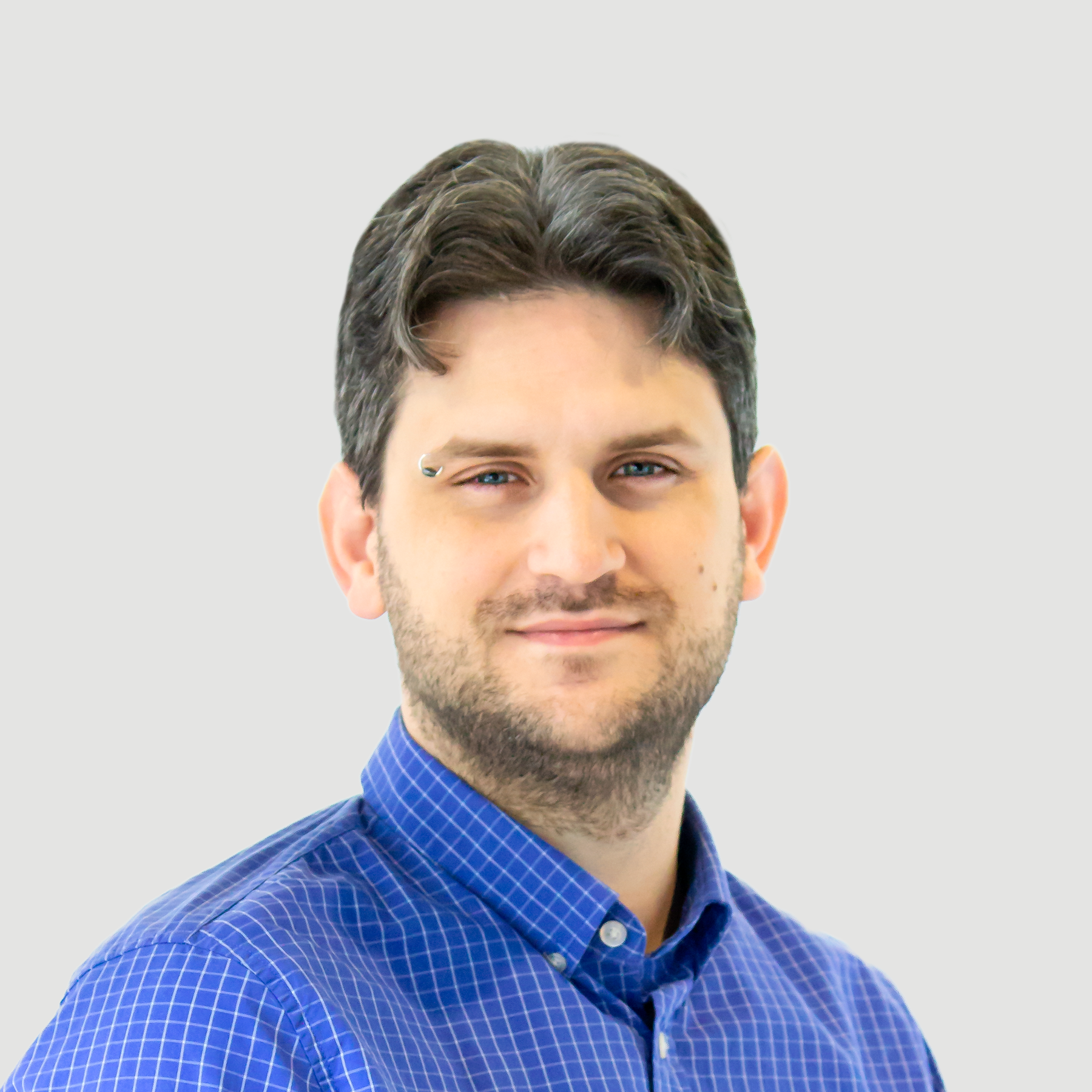}}]{Marcus A.\ Brubaker}
Marcus A.\ Brubaker received his PhD in 2011 at the University of Toronto. He did postdocs at the Toyota Technological Institute at Chicago, Toronto Rehabilitation Hospital and the University of Toronto. His research interests span computer vision, machine learning and statistics. Dr.\ Brubaker is an Assistant Professor at York University (Toronto, Canada), an Adjunct Professor at the University of Toronto and a Faculty Affiliate of the Vector Institute.
He is also an Academic Advisor to Borealis AI where he previously worked as the Research Director of the Toronto office.
He is currently an Associate Editor for the journal IET Computer Vision and has served as a reviewer and area chair for many computer vision and machine learning conferences.
\end{IEEEbiography}

\end{document}